\documentclass[11pt]{article}
\usepackage{amsmath, amssymb, amsthm, color, epsfig, epstopdf, enumerate,framed}
\usepackage{graphicx, hyperref, listings, mathtools, multicol, multirow, pdfpages, subfigure}
\usepackage{booktabs, siunitx} 
\usepackage{textgreek}

\usepackage{amsopn, authblk} 
\usepackage[table]{xcolor}
\usepackage[numbers,sort&compress]{natbib} 
\usepackage{appendix}

\usepackage{algorithm, algpseudocode} 

\theoremstyle{plain}
\newtheorem{assumption}{Assumption}
\newcommand{\A}{\boldsymbol{A}} 
\newcommand{\bc}{\boldsymbol{c}} 
\newcommand{\bh}{\boldsymbol{h}}
\newcommand{\bx}{\boldsymbol{x}}
\newcommand{\X}{\boldsymbol{X}}
\newcommand{\bomega}{\boldsymbol{\omega}}
\newcommand{\by}{\boldsymbol{y}}

\newcommand{\R}{\mathbb{R}}
\newcommand{\argmin}[1]{\underset{#1}{\mathrm{argmin}}}
\newcommand{\card}{{\rm card}}

\newcommand{\supp}{{\rm supp}}

\newcommand{\vertiii}[1]{{\left\vert\kern-0.25ex\left\vert\kern-0.25ex\left\vert #1 
    \right\vert\kern-0.25ex\right\vert\kern-0.25ex\right\vert}}

\theoremstyle{plain}
\newtheorem{theorem}{Theorem}[section]
\newtheorem*{theorem*}{Theorem}
\newtheorem{lemma}[theorem]{Lemma}

\theoremstyle{definition}

\theoremstyle{remark}

\newtheorem{remark}[theorem]{Remark}

\numberwithin{equation}{section}
\numberwithin{algorithm}{section}
\numberwithin{figure}{section}
\numberwithin{table}{section}

\title{Conditioning of Random Feature Matrices: \\Double Descent and Generalization Error}
\author{Zhijun Chen}
\author{Hayden Schaeffer}
\affil{Department of Mathematical Sciences, Carnegie Mellon University, Pittsburgh, PA 15213. (\text{zhijunc@andrew.cmu.edu}, { }\text{schaeffer@cmu.edu})}
\date{2021}

\begin{document}
\maketitle

\begin{abstract}
We provide (high probability) bounds on the condition number of random feature matrices. 
In particular, we show that if the complexity ratio $\frac{N}{m}$ where $N$ is the number of neurons and $m$ is the number of data samples scales like $\log^{-1}(N)$ or $\log(m)$, then the random feature matrix is well-conditioned. This result holds without the need of regularization and relies on establishing various concentration bounds between dependent components of the random feature matrix. Additionally, we derive bounds on the restricted isometry constant of the random feature matrix.  We prove that the risk associated with regression problems using a random feature matrix exhibits the double descent phenomenon and that this is an effect of the double descent behavior of the condition number. The risk bounds include the underparameterized setting using the least squares problem and the overparameterized setting where using either the minimum norm interpolation problem or a sparse regression problem. For the least squares or sparse regression cases, we show that the risk decreases as $m$ and $N$ increase, even in the presence of bounded or random noise. The risk bound matches the optimal scaling in the literature and the constants in our results are explicit and independent of the dimension of the data.

 \end{abstract}

\section{Introduction}
\label{sec: intro}

Random feature methods are essentially two-layer shallow networks whose single hidden layer is randomized and not trained \cite{rahimi2007random, rahimi2008uniform}.
They can be used for various learning tasks, including regression, interpolation, kernel-based classification, etc. Studying their behavior not only results in a deeper understanding of kernel machines but could also yield insights into more complex models, such as very wide or deep neural networks. 

 Consider $\bx\in \R^{d}$ to be the input vector and $\boldsymbol{W}=[\boldsymbol{\omega}_{i,j}]\in \R^{d \times N}$ to be a random weight matrix, where each column of the matrix $\boldsymbol{W}$ is sampled from some prescribed probability density $\rho(\boldsymbol{\omega})$. Then the random feature map is defined by $\phi(\boldsymbol{W}^T\X) \in \mathbb{C}^N$ where $\phi:\R\rightarrow \mathbb{C}$ is some Lipschitz activation function applied element-wise to the input. The resulting two-layer network with $N$ neurons is defined by $\by=\boldsymbol{C}^T\phi(\boldsymbol{W}^T\X) \in \mathbb{C}^ {d_{out}}$ where the final layer $\boldsymbol{C}\in\mathbb{C}^{N \times d_{out}}$ needs to be trained for a given problem. The analysis of these models with respect to the dimensionality, stability, testing error, etc. often depends on the extreme singular values (or the condition number) of the matrix $\A:=\phi(\X^T \boldsymbol{W}) \in \mathbb{C}^{m\times N}$, where $\X=[\bx_1, \ldots \bx_m]\in\mathbb{R}^{d\times m}$ is a collection of $m$ random samples of the input variable. In particular, for regression or interpolation, knowledge of the condition number of $\A$ can yield control over the error as well as reveal the landscape of the testing error as a function of the model complexity ratio $\frac{N}{m}$.

The testing error (or risk) of random feature models have been of recent interest \cite{rahimi2008weighted, rudi2017generalization,li2019towards,weinan2020towards, hashemi2021generalization,mei2021generalization}; in particular, with a goal of quantifying the number of features needed to obtain a given learning rate. In \cite{rahimi2008weighted}, it was shown that the random feature model yields a test error of $\mathcal{O}(N^{-\frac{1}{2}}+m^{-\frac{1}{2}})$ when trained on Lipschitz loss functions. Thus if $m \asymp N$ then the generalization error is $\mathcal{O}(N^{-\frac{1}{2}})$ for large $N$.  In \cite{rudi2015less},  a comparison between the test error using kernel ridge regression with both the full kernel matrix and a column subsampled kernel matrix  (related to Nystr\"{o}m's method) suggests that the number of columns (and rows) should be roughly the square root of the number of original samples in order to obtain their learning rate. The squared loss case was consider in \cite{rudi2017generalization,li2019towards}, providing risk bounds when the random feature model is trained via ridge regression. In \cite{rudi2017generalization}, it was shown that for $f$ in an RKHS, using $N=\mathcal{O}(\sqrt{m} \, \log{m})$ is sufficient to achieve a test error of $\mathcal{O}(m^{-\frac{1}{2}})$ with a random feature method.  The results in \cite{rudi2015less,rudi2017generalization}  use similar technical assumptions on the kernel and the second moment operator, e.g. a certain spectral decay rate on second moment operator, which may be difficult to verify. Note that for the ridge regression results, the penalty parameter $\lambda=\lambda(m)$ often must remain positive and thus such results do not include interpolation. In \cite{weinan2020towards}, the risk was analyzed for a regularized model with the target function in an RKHS, noting that to achieve $N^{-1} + m^{-\frac{1}{2}}$ one must place explicit assumptions on the decay rate of the spectrum of the kernel operator, see also \cite{bach2017equivalence}. In \cite{mei2021generalization}, an investigation of the kernel ridge regression and the random feature ridge regression problems in terms of the $N$, $m$, and $d$ was given. The main result suggests that an `optimal' choice for the complexity ratio $\frac{N}{m}$ depends on an algebraic scaling in $N$, $m$, and $d$.

In the overparameterized setting, minimum norm regression problems are often used, either by minimizing the $\ell^2$ or $\ell^1$ norm with a data fitting constraint. The $\ell^1$ problem can be used to obtain a low complexity random feature model in the highly overparameterized setting and was proposed and studied in \cite{hashemi2021generalization} (see also \cite{yen2014sparse} for a related algorithmic approach). By solving the $\ell^1$ basis pursuit denoising problem with a random feature matrix,  test error bounds where obtain as a function of the sparsity, $N$, $m$, and $d$. It is worth noting that in the dense case, i.e. when the sparsity equals $N$ and $m\asymp N^2$ up to log terms, the test error scales like $N^{-1} + m^{-\frac{1}{2}}$ up to log terms, thus the method achieves the upper bounds provided by other methods but utilizes a different computational approach.

The min-norm interpolation problem (or the ridgeless limit as $\lambda\rightarrow 0^+$)  has received recent attention for random feature methods \cite{belkin2019reconciling, mei2019generalization, belkin2019does, bartlett2020benign, hastie2019surprises, tsigler2020benign, liang2020just, mei2021generalization}. In \cite{bartlett2020benign,tsigler2020benign}, the test error for general ridge and ridgeless regression problems is shown to be controlled by the condition number, or the related notion of effective rank, of the random design matrix. It was shown in \cite{mei2021generalization} that the min-norm interpolators are optimal among all kernel methods in the overparameterized regime. 

In some settings, overparameterized random feature models can produce lower test error even though they have large complexity. However, choosing $N$ and $m$ only using the risk bounds could lead to an insufficient picture of the parameter landscape for random feature models or neural network. In modern learning, the graph of the test error as a function of the model's complexity (i.e. relative number of neurons) has two valleys. The first is when the number of parameters are relatively small, this is the underparameterized regime where the model has small complexity and small risk. As the complexity reaches the interpolation threshold ($N=m$), the risk tends to peak and then begins to decay again as the number of parameters continues to increase well into the overparameterized regime. This is referred to as the double descent phenomenon \cite{belkin2018understand, belkin2019reconciling, belkin2020two,advani2020high} and has had several recent results, specifically, in the characterization of the test error as a function of the complexity ratio \cite{kan2020avoiding,belkin2018understand, belkin2019reconciling, belkin2020two,advani2020high,mei2019generalization,mei2021generalization, ba2019generalization}. 
A direct analysis of the double descent curves for random feature regression \cite{mei2019generalization,mei2021generalization} showed that in the overparameterized regime, taking $N/m=\mathcal{O}(1)$ can lead to suboptimal results which can be corrected if $N/m\rightarrow \infty$. To achieve near optimal test errors, they suggest setting $N\asymp m^{1+\delta}$ for small $\delta>0$.

\subsection{Some Related Work}
Some of the earlier results in the literature on random matrix theory for some related problems focused on random kernel matrices in the form $K_{i,j} = \phi(\bx_i^T\bx_j)$, specifically,  characterizing the spectrum of square matrices depending on one random variable $\bx$ \cite{el2010spectrum, cheng2013spectrum, fan2019spectral}. In \cite{pennington2019nonlinear, pennington2017resurrecting, hastie2019surprises,benigni2019eigenvalue, pastur2020random}, the asymptotic behavior of asymmetric rectangular random matrices $\A=\phi(\X^T\boldsymbol{W})$ were studied, in particular, to quantify the limiting distribution on the spectrum of the Gram matrix $\A\A^*$. In \cite{louart2018random} the asymptotic behavior of the resolvent operator of the Gram matrix was studied, where $\boldsymbol{W}$ is random but $\X$ is deterministic. Using concentration arguments, \cite{liao2018spectrum,louart2018random} analyzed the spectrum of the Gram matrix of random feature maps in the high-dimensional setting where $N$ grows like $m$. Precise asymptotics for the random feature ridge regression model in the large $m$, $N$, and $d$ limit and a characterize of the test error as function of the dimensional parameters were given in \cite{liao2020random}. They also showed the existence of a phase transition with respect to $N/m$  at the interpolation threshold.  In \cite{liang2020multiple}, the authors studied the multiple-descent curve for the scaling $d=m^{\alpha}$ where $\alpha \in (0,1)$. They observed non-monotonic behavior in the test error as a function of the sample size $m$. The results in \cite{liang2020multiple} rely on the restricted lower isometry of the kernels, showing that (with high probability) the empirical kernel matrix has a certain number of nonzero eigenvalues with a lower bound.

\subsection{Our Contributions}
In this work, we provide (high probability) bounds on the conditioning of the random feature matrix $\A=\phi(\X^T\boldsymbol{W})$ for $N$ neurons, $m$ samples, and with dimension $d$. In particular, when the  $N/m$ scales like $\log(m)$ (overparametrized) or $\log^{-1}(N)$ (underparameterized), the singular values of $\A$ concentrate around 1 and thus with high probability the condition number of the random feature matrix is small. In addition, the guarantees break down in the interpolation region where we prove that the system is ill-conditioned. When $N=m$, as the number of features (or equivalently the number of samples) increases the conditioning at the interpolation regime worsens. This transition between the conditioning states coincides with the phase transition in the double descent curves (the risk) for random feature regression. As an additional application of our results, we  show that the condition number can be used to obtain state-of-the-art test error bounds for a particular class of target functions. We highlight some comparisons with other works below.

\begin{itemize}
\item Our theoretical results are related to other random matrix results, such as \cite{pennington2019nonlinear, pennington2017resurrecting, hastie2019surprises,benigni2019eigenvalue, pastur2020random, louart2018random}; however, we quantify the behavior of the random feature matrix for both finite and large $N$ and $m$ and provide both sample and feature complexity bounds. Our results give a simple scaling between $N$ and $m$ that yields random feature matrices which are nearly isometries. 
\item We connect the conditioning of the random matrix to the double descent curves and provide another characterization of the underparameterized and overparameterized regions, compared to \cite{mei2019generalization,mei2021generalization}. Specifically, our results show that the complexity ratio $N/m$ need only be logarithmically far from the interpolation threshold to have small risk, rather than algebraically far. Our results support similar conclusions to \cite{mei2019generalization,mei2021generalization}.
\item The proofs give explicit values for the constants used in our bounds and thus our results hold up to additive terms, for example, compared to the results of \cite{rudi2017generalization} which are based on overall rates.  Our bounds improve for high-dimensional data, since the constants are independent of the dimension $d$.
\item The analysis in \cite{hashemi2021generalization} established the structure of the random feature matrix using a coherence estimate (for sparse regression in the overparameterize regime). A coherence bound leads to a non-optimal quadratic scaling (up to log terms)  between $m$ and $N$ (or the sparsity $s$). Our results directly estimate the restricted isometry constant thus yielding a linear scaling (up to log terms) between $m$ and $N$ (or $s$). This difference allows for a more precise characterization of the underparameterized and overparameterized regimes, i.e. it shows that the complexity ratio only needs to scale by log factors rather than polynomials. Our results also match the transition point $N=m$ seen in \cite{mei2019generalization,mei2021generalization, liao2020random} and the linear scaling considered in \cite{pennington2019nonlinear}. 
\item We establish risk bounds for random feature regression without the need for assumptions on the spectrum of the second moment operator or kernel \cite{rudi2015less,rudi2017generalization}. We also improve the bounds in the sparse setting \cite{hashemi2021generalization}. Similar results hold for other conditioning based methods, such as greedy solvers. Our risk bounds include the effects of noise on the measurements, with the assumption that the noise is bounded (or bounded with high probability, e.g. normally distributed). 
\end{itemize}

\subsection{Notation}
Let $\mathbb{R}$ be the set of all real number and $\mathbb{C}$ be the set of all complex number where $i=\sqrt{-1}$ denotes the imaginary unit. Define the set $[N]$ to be all natural numbers smaller than $N$, i.e. $\{1,2,\dots, N\}$. Throughout the paper, we denote vectors or matrices with bold letters, and denote the  identity matrix of size $n\times n$ by $\boldsymbol{I}_n$. For two vectors $\X, \boldsymbol{y} \in \mathbb{C}^d$, the inner product is denoted by $\langle \X, \boldsymbol{y} \rangle = \sum_{j=1}^{d} x_j\overline{y}_j$  where $\X = [x_1,\dots,x_d]^T$ and $\boldsymbol{y}=[y_1,\dots,y_d]^T$. For a vector $\X\in \mathbb{C}^d$, we denote by $\|\X\|_p$ the $\ell^p$-norm of $\X$ and for a matrix $\A\in\mathbb{C}^{m\times N}$ the (induced) $p$-norm is written as $\|\A\|_p$. For a matrix $\A\in\mathbb{C}^{m\times N}$, its transpose is denoted as $\A^T$ and its conjugate transpose is denoted as $\A^*$. Lastly, let $\mathcal{N}(\boldsymbol{\mu},\boldsymbol{\Sigma})$ be the normal distribution with mean vector $\boldsymbol{\mu}\in \mathbb{R}^d$ and covariance matrix $\boldsymbol{\Sigma} \in \mathbb{R}^{d\times d}$. If $\A\in\mathbb{C}^{m\times N}$ is full rank, then the pseudo-inverse is  $\A^\dagger= \A^*(\A\A^*)^{-1}$ if $m\leq N$ or $\A^\dagger= (\A^*\A)^{-1}\A^*$ if $m\geq N$. The condition number with respect to the $2$-norm of a matrix $\A$ is defined by $\kappa(\A):=\|\A\|_2 \|\A^\dagger\|_2$. Denote the $k$-th eigenvalue of the matrix $\boldsymbol{M}$ by $\lambda_k(\boldsymbol{M})$.

\section{Empirical Behavior}
\label{sec: tests}

The double descent phenomena appears in both the condition number for random matrices and the risk associated with the corresponding regression problem. In Figure \ref{figure: risk}, we examine the behavior of the risk and the condition number of the Gram matrix as a function of the complexity $\frac{N}{m}$, where $N$ is the column size, i.e. the number of features, and $m$ is the row size, i.e. the number of samples. Let $d$ be the dimension of each sample. We set $d=3$ and $m=100$, and vary the number of features $N\in[500]$. The weights of the random feature matrix are sampled from $\mathcal{N}(0,0.1)$ and the 100 training points are sampled from $\mathcal{N}(0,1)$.  The target function for the $\ell^2$ regression problem is linear $f(\bx) = \boldsymbol{b}^T\bx$ where the components of $\boldsymbol{b}\in \mathbb{R}^d$ are randomly sampled from the uniform distribution $U[0,1]^d$, which differs from the randomness used in the weights. The outputs on the training samples have either zero noise or additive gaussian noise with $10\%$ SNR (as indicated on the graphs). The training is done using the psuedo-inverse, i.e. the least squares problem or the min-norm interpolator depending on the complexity ratio. The empirical risk is measured on 1000 testing samples from $\mathcal{N}(0,1)$. For each $N$, the risk and condition numbers are averaged over 10 random trials. Each curve in Figure \ref{figure: risk} is rescaled so that the values are within the same range, noting that the minimum value of the condition number is 1 before rescaling. In each case, the double descent of the condition number and the empirical risk coincides with the same complexity regimes. Specifically, both peak at the interpolation threshold $\frac{N}{m}=1$ with large values obtained within a certain width of this threshold, which was also observed for this problem in \cite{ma2020slow}. It is worth noting that in these three experiments, using Fourier features \cite{rahimi2007random, rahimi2008weighted, rahimi2008uniform} with and without noise or using ReLU features, the global minima of the risk is obtain in the overparameterized regime, i.e. when $\frac{N}{m}>1$.

\begin{figure}[t]
\begin{center}
\includegraphics[trim={0.1cm 0.1cm 0.1cm 0.1cm}, width=5.5in, clip]{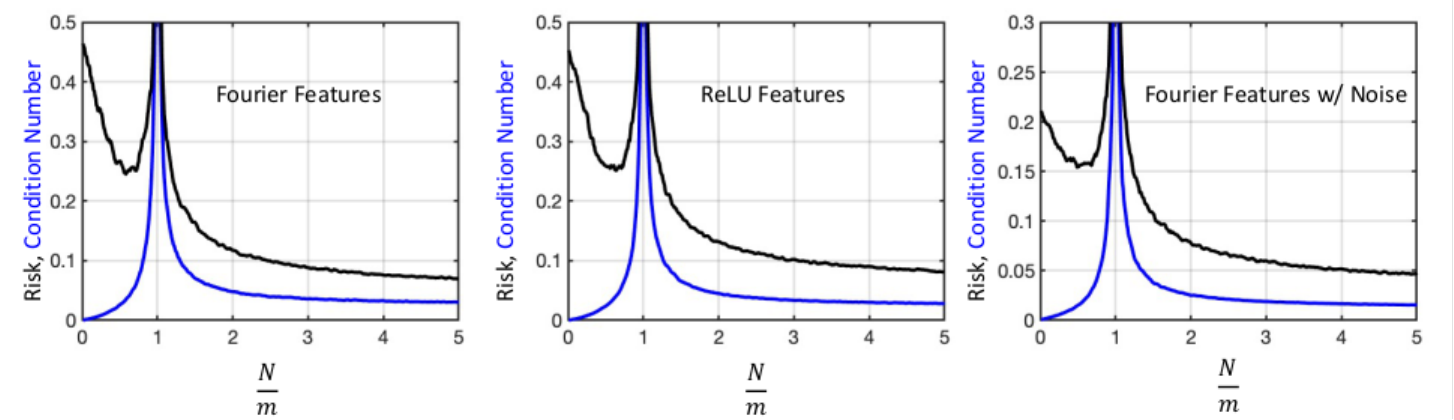}
\caption{\textbf{Double Descent of the Condition Number and Risk}. The condition number and risk curves are plotted as a function of the complexity ratio $\frac{N}{m}$. The curves are rescaled for visualization purposes. Each of the graphs display the double descent phenomena corresponding to the same complexity regions. In each plot, the global minima of the risk is obtain in the overparameterized regime  $\frac{N}{m}>1$ where the linear system remains well-conditioned as $N$ grows. } \label{figure: risk}
\end{center}
\end{figure}

In Figure \ref{figure: singular}, the probability density function associated with the singular values of the normalized random feature matrix in the underparameterized and overparameterized regimes are plotted using different complexity scales $\frac{N}{m}$. The curves in Figure \ref{figure: singular} are normalized to have their maxima equal to 1 and are estimated using a smooth density approximation over 10 trials. In both plots, the red curves indicate the interpolation threshold $\frac{N}{m}=1$ and are the transition boundaries between the two regimes (i.e. between the two plots). The weights and data are sampled from $\mathcal{N}(0,1)$. In this experiment we set $d=50$ and $m=3d=150$, and vary the number of features $N$ based on the logarithmic scalings indicated on the graphs.  
In Figure \ref{figure: singular}, the corresponding linear system is well-conditioned when the support of the probability density is bounded away from the endpoints of the interval. Empirically, this is the case (with high probability) for the log and log-cubed scaling. 

\begin{figure}[t]
\begin{center}
\includegraphics[trim={0.1cm 0.1cm 0.1cm 0.1cm}, width=5.5in, clip]{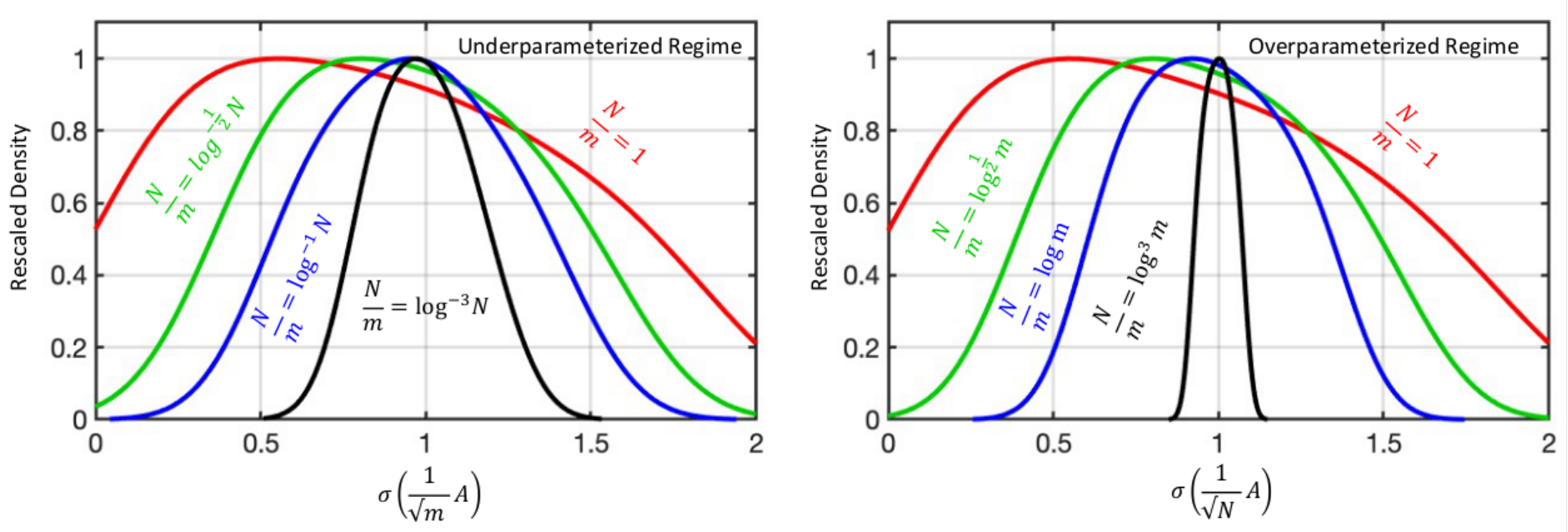}
\caption{\textbf{Probability Density Functions for the Singular Values:} The (normalized) smooth estimations of the probability density functions associated with the singular values of the normalized random feature matrix in the underparameterized (left) and overparameterized (right) regimes are plotted. The red curves indicate the interpolation case $\frac{N}{m}=1$ and are the transition boundaries between the two plots. }\label{figure: singular}
\end{center}
\end{figure}

The choice of scalings in the plots are to provide empirical support for the theoretical results in Section \ref{sec: theory}, specifically, that a logarithmic scaling is sufficient to obtain a bound on the extreme singular values. In addition, there is an observable symmetry between the (estimated) probability density functions in the underparameterized and overparameterized regimes using the same scaling (with $m$ and $N$ switched). The symmetry is not exact, partly due to rounding in order to obtain integer values for $N$, but the plots in Figure \ref{figure: singular} show that the overparameterized and underparameterized systems are well-conditioned when using the same functional scaling.

\section{Condition Number of Random Feature Maps and Insights}
\label{sec: theory}

In this section, we discuss the bounds on the extreme singular values of the random feature matrix $\A=\phi(\X^T\boldsymbol{W})$ where the elements of the matrices $\boldsymbol{W}$ and $\X$ are randomly drawn from normal distributions.

\subsection{Main Result}
Before discussing the results, we state the main assumption that is used throughout the paper. Specifically, the data and weights are normally distributed and the random feature matrix is built using Fourier features. This assumption is not necessary for the conclusions in this work; however, they simplify expressions for the constants and the parameter scalings. 

\begin{assumption} \label{assumption}
The data samples $\{\bx_j\}_{j\in [m]}$ are drawn from $\mathcal{N}(\boldsymbol{0},\gamma^2 \boldsymbol{I}_d)$ and the set of random feature weights $\{\boldsymbol{\omega}_k\}_{k\in [N]}$ are drawn from $\mathcal{N}(\boldsymbol{0},\sigma^2 \boldsymbol{I}_d)$. The random feature matrix $\A\in\mathbb{C}^{m\times N}$ is defined component-wise by $a_{j,k} = \phi(\bx_j,\boldsymbol{\omega}_k)$, where $\phi(\X,\boldsymbol{\omega}) = \exp(i \langle \bx,\boldsymbol{\omega} \rangle)$. 
\end{assumption}

\begin{theorem}[\textbf{Conditioning in Three Regimes}]\label{Th: Main Theorem (simplified)}
 Let the data $\{\bx_j\}_{j\in [m]}$, weights $\{\boldsymbol{\omega}_k\}_{k\in [N]}$, and random feature matrix $\A\in\mathbb{C}^{m\times N}$ satisfy Assumption \ref{assumption}.  \begin{itemize}
    \item [(a)](\textbf{Underparameterized Regime $m>N$}) If the following conditions hold
    \begin{align}
       &m\geq C\eta^{-2}N\log\left(\frac{2N}{\delta}\right)\\
         &\sqrt{\delta}\,\eta \, (4\gamma^2 \sigma^2+1)^{\frac{d}{4}} \geq N \label{ineq: simplified uncertainty principle}
    \end{align}
 for some $\eta,\delta>0$, then with probability at least $1-2\delta$ we have 
    $$\left| \lambda_k\left(\frac{1}{m}\A^*\A\right)-1\right| \leq 3\eta,$$
    for all $k\in [N]$.
    \item [(b)](\textbf{Overparameterized Regime $m<N$}) If the following conditions hold
    \begin{align*}
       & N\geq C\eta^{-2}m\log\left(\frac{2m}{\delta}\right)\\
       &\sqrt{\delta}\,\eta \,(4\gamma^2 \sigma^2+1)^{\frac{d}{4}} \geq m
    \end{align*}
    for some $\eta,\delta>0$, then with probability at least $1-2\delta$ we have 
    $$\left| \lambda_k\left(\frac{1}{N}\A\A^*\right)-1\right| \leq 3\eta,$$
    for all $k\in [m]$.
    \item [(c)](\textbf{Interpolation Threshold}) If $m=N\geq 2$,
then the eigenvalues satisfy 
    \begin{align*}
          &\mathbb{E} \, \lambda_{\min} \left(\frac{1}{N} \A^*\A\right) \leq \left(1-\frac{1}{N}\right)^{\frac{1}{2}}\frac{1}{(4\gamma^2\sigma^2+1)^{d/4}}+ \frac{1}{N}\\
        &\mathbb{E}\lambda_{\max} \left(\frac{1}{N} \A^*\A\right)\geq   2-\frac{1}{N},\end{align*}
where the expectation is with respect to $\{\bx_j\}_{j\in [m]}$ and $\{\boldsymbol{\omega}_k\}_{k\in [N]}$.
\end{itemize}
The universal constant $C>0$ is no greater than $4$ if $\min\{m,N\}>25$.
\end{theorem}
The proof of Theorem \ref{Th: Main Theorem (simplified)} is given in Section \ref{proof of condition}.

Theorem \ref{Th: Main Theorem (simplified)} provides a characterization of the behavior of the condition number of the random feature matrix as a function of complexity $\frac{N}{m}$, with respect to any dimension $d$. If the parameter $\eta$ is chosen to satisfy $\eta < \frac{1}{3}$, then the eigenvalues of the normalized Gram matrices are strictly between $0$ and $2$. Moreover, in the extreme cases, when $m \gg N$ or $m \ll N$, then $\eta$ can be small and the eigenvalues of the normalized Gram matrices will concentrate near $1$. If the random feature matrix is square and we assume that the complexity bound $m=N\leq(4\gamma^2\sigma^2+1)^{d/4}$ holds, then by applying Markov's inequality the minimum eigenvalue of the matrix $\frac{1}{N}\A^* \A$ satisfies
$$\mathbb{P}\left(\lambda_{\min} \left(\frac{1}{N} \A^*\A\right)\geq N^{-\frac{1}{2}} \right) \leq 2N^{-\frac{1}{2}}. $$ On the other hand, the largest eigenvalue of $\frac{1}{N}\A^*\A$ must be larger than $1$. Therefore, at the interpolation threshold, the Gram matrix becomes ill-conditioned. Altogether, Theorem \ref{Th: Main Theorem (simplified)} shows that the conditioning of the random feature matrix exhibits the double descent phenomena. 

One interesting consequence of Theorem \ref{Th: Main Theorem (simplified)} is that the random feature matrix does not need regularization to be well-conditioned. This was also observed in \cite{poggio2019double} for random matrices $\A$, where $a_{i,j}$ are drawn randomly from some probability distribution and for kernel matrices $\A=K(\X^T\X)$, i.e. symmetric systems.

\begin{remark} 
The symmetry in the conditions for the overparameterized and underparameterized regimes is inherited from the assumption that the weights and samples use the same distribution (with different variances). The results can be extended to uniform or sub-gaussian random sampling for $\bx$ or $\boldsymbol{\omega}$ and other feature maps, which will lead to similar overall conditions but break this type of symmetry.  
\end{remark}

\begin{remark} 
In several of the theoretical results, the smaller of the dimensional parameter ($N$, $m$, or sparsity $s$) must be less than $\sqrt{\delta}\,\eta \,(4\gamma^2 \sigma^2+1)^{\frac{d}{4}}$. If we set the parameters to $\delta=0.01$, $\eta=\frac{1}{4}$, and $\gamma,\sigma=2$, then for $d\geq 44$ the complexity limitation is on the exascale. If one wants the inner product $\langle \bx ,\boldsymbol{\omega} \rangle$ to be order 1, then after rescaling we get  $\left(\frac{4\gamma^2\sigma^2}{d}+1\right)^{d/4}$, which is like $\exp(\gamma^2\sigma^2)$ for large $d$. Thus in the rescaled case in high-dimensions, if $\gamma\sigma=\sqrt{46}$, then the complexity limitation is on the exascale. On the other hand, for small $\gamma\sigma$ the complexity bounds can place a severe limitation on the smaller of the dimensional parameters. 
\end{remark}

\section{Generalization Error using the Condition Number} \label{sec:gen}
With the control on the condition number, we can provide bounds on the generalization error (or risk) associated with regression problems using random feature matrices in both the underparameterized and overparameterized settings. We show that the double descent phenomena on the risk is a consequence of the double descent phenomena on the condition number. 

\subsection{Set-up and Notation}
The analysis for the generalization error will follow the set-up from \cite{rahimi2007random, rahimi2008weighted, rahimi2008uniform, hashemi2021generalization}. For a probability density $\rho$ (associated with the random weights $\boldsymbol{\omega}$) in $\mathbb{R}^d$ and an activation function $\phi:\mathbb{R}^d \times \mathbb{R}^d \to \mathbb{C}$, a function $f:\mathbb{R}^d \to \mathbb{C}$ has finite $\rho$-norm with respect to $\phi$ if it belongs to the class \cite{rahimi2007random}
\begin{equation*}
    \mathcal{F}(\phi,\rho): =\left\{ f(\bx) = \int_{\mathbb{R}^d} \alpha(\boldsymbol{\omega}) \phi(\bx;\boldsymbol{\omega}) d\boldsymbol{\omega}: \quad \|f\|_\rho :=\sup_{\boldsymbol{\omega}} \left| \frac{\alpha(\boldsymbol{\omega})}{\rho(\boldsymbol{\omega})} \right|<\infty \right\}.
\end{equation*}
The regression problem is to find an approximation of a target function $f\in \mathcal{F}(\phi,\rho)$, that is, given a set of random weights $\{\boldsymbol{\omega}_k\}_{k\in [N]}$, we train
\begin{equation}\label{eq: approximate f with linear combinations}
    f^\sharp(\bx) = \sum_{k=1}^N c^\sharp_k \,\phi(\bx,\boldsymbol{\omega}_k).
\end{equation}
using $\left(\bx_j, y_j \right)$ where $y_j =f(\bx_j)+e_j$ and $e_j$ is the measurement noise, for $j\in[m]$. Let $\A\in\mathbb{C}^{m\times N}$ be the matrix with $a_{j,k}=\phi(\bx_j,\boldsymbol{\omega}_k)$ for $j\in[m]$ and $k\in[N]$, the training problem is equivalent to finding $\bc\in\mathbb{C}^{N}$ such that  $\A\bc \approx \boldsymbol{y}$ where $\boldsymbol{y} = [y_1,\dots,y_m]^T\in\mathbb{C}^m$.

\begin{assumption} \label{assumption2}
The noise on outputs, $\{e_j\}_{j\in [m]}$, is bounded by some constant $E$, i.e. $\|e_j\|_\infty\leq E$.
\end{assumption}

Assumption \ref{assumption2} can be extended to random noise. For example, if $e_j\sim \mathcal{N}(\boldsymbol{0},\nu^2 \boldsymbol{I}_d)$ and $m\geq 2 \log\left( \frac{1}{\delta} \right)$ then with probability at least $1-\delta$ the noise is bounded by $|e_j|\leq 2\nu$ for all $j\in[m]$. This does not change our technical arguments except for an additional $\delta$ term in the probability bounds.

\subsection{Underparameterized Regime, Least Squares Problem}
In the underparameterized (overdetermined) regime where we have more data samples than features, the coefficient $\bc^{\sharp}$ are trained via the least squares problem
\begin{align}
    \bc^{\sharp} = \argmin{\bc\in \mathbb{R}^N }  \ \|\A\bc - \boldsymbol{y}\|_2,
\end{align}
where $\bc^{\sharp}=[c^{\sharp}_1,\dots,c^{\sharp}_N]^T$ and the (trained) approximation is given by
\begin{equation}\label{eq: def of f-sharp}
    f^{\sharp}(\bx) = \sum_{k=1}^N c^{\sharp}_k\, \phi(\bx,\boldsymbol{\omega}_k).
\end{equation}
Using the condition number of $\A$, we can control the risk
$$R(f^{\sharp}):=\|f-f^{\sharp}\|^2_{L^2(d\mu)}=\int_{\mathbb{R}^d} \left|f(x)-f^{\sharp}(x)\right|^2 \, d\mu(x)$$
with high probability.

\begin{theorem}[\textbf{Least Squares Risk Bound $m>N$}]\label{Th: error for least square m>N}
Let the data $\{\bx_j\}_{j\in [m]}$, weights $\{\boldsymbol{\omega}_k\}_{k\in [N]}$, and random feature matrix $\A\in\mathbb{C}^{m\times N}$ satisfy Assumption \ref{assumption} and let the noise satisfy Assumption \ref{assumption2}. If for some $\eta, \delta>0$ the following conditions hold 
\begin{align*}
    &m\geq C\eta^{-2}N\log\left(\frac{2N}{\delta}\right)\\
        &\sqrt{\delta}\eta \,  \left(4\gamma^2 \sigma^2+1\right)^{\frac{d}{4}} \geq N
\end{align*}
where $C>0$ is a universal constant independent of the dimension $d$, then with probability at least $1-8\delta$ the following risk bound holds:
\begin{align}\label{eq: LS risk bound}
    R(f^{\sharp})\leq 16\, K(\eta) \, \left(1+Nm^{-\frac{1}{2}}\log^{\frac{1}{2}}\left(\frac{1}{\delta}\right)\right)\, (\epsilon^2 \|f\|^2_{\rho} + E^2),
\end{align}
where $K(\eta):= \frac{1+3\eta}{1-3\eta}$ and
\begin{align*}
    \epsilon = \frac{2}{\sqrt{N}} \left( 1+4\gamma\sigma d \sqrt{1+\sqrt{\frac{12}{d}\log\left(\frac{m}{\delta}\right)}} +\sqrt{\frac{1}{2}\log \left(\frac{1}{\delta}\right)}\right).
\end{align*}
\end{theorem}
The proof of Theorem \ref{Th: error for least square m>N} is given in Section \ref{proof generalization}. Note that $K(\eta)$ is an upper bound for the condition number of the matrix $\frac{1}{m}\A^*\A$ using Theorem \ref{Th: Main Theorem (simplified)}.

\subsection{Overparameterized Regime, Min-Norm Interpolator}
In the overparameterized (underdetermined) regime, we have more features than data, and thus popular models for training the coefficient vector $\bc^{\sharp}$ rely on minimizing $\|\bc^{\sharp}\|_p$ under the constraint that $\A\bc^{\sharp} - \boldsymbol{y}$ is zero or small. First, we consider training $\bc^{\sharp}$ using the min-norm interpolation problem:
\begin{align}\label{min-norm problem}
    \bc^{\sharp} =\argmin{\A\bc=\boldsymbol{y} } \  \ \|\bc\|_2 
    \end{align}
The solution is given by $\bc^{\sharp}=\A^\dagger\boldsymbol{y}$ where the psuedoinverse is  $\A^\dagger= \A^*(\A\A^*)^{-1}$, noting that the inverse of the Gram matrix exists with high probability under the conditions of Theorem \ref{Th: Main Theorem (simplified)}. This is also referred to as the ridgeless case, since $\bc^{\sharp}=\A^\dagger\boldsymbol{y}$ can be viewed as the ridgeless limit of the coefficients, $\bc_\lambda^{\sharp}$, which are obtained by the ridge regression problem 
\begin{align}\label{ridge regression problem}
    \bc_\lambda^{\sharp} =\argmin{\bc\in\mathbb{R}^{N}  } \ \frac{1}{m}\| \A\bc-\boldsymbol{y}\|^2_2+ \lambda \|\bc\|^2_2 
    \end{align}
    i.e.  the vector obtained when $\lambda\rightarrow 0^+$. With $f^{\sharp}$ defined as \eqref{eq: def of f-sharp}, then we have the following result analogous to Theorem \ref{Th: error for least square m>N}.

\begin{theorem}[\textbf{Min-Norm Risk Bound $m<N$}]\label{Th: error for least square m<N}
Let the data $\{\bx_j\}_{j\in [m]}$, weights $\{\boldsymbol{\omega}_k\}_{k\in [N]}$, and random feature matrix $\A\in\mathbb{C}^{m\times N}$ satisfy Assumption \ref{assumption} and let the noise satisfy Assumption \ref{assumption2}. If for some $\eta, \delta>0$ the following conditions hold 
\begin{align*}
    &N\geq C\eta^{-2}m\log\left(\frac{2m}{\delta}\right)\\
    &\sqrt{\delta}\eta \, \left(4\gamma^2 \sigma^2+1\right)^{\frac{d}{4}} \geq m
\end{align*}
where $C>0$ is a universal constant independent of the dimension $d$, then there exists a constant $\tilde{C}>0$ such that the following risk bound holds:
\begin{align*}
    R(f^{\sharp}) &\leq \tilde{C}\log^{\frac{1}{2}}\left(\frac{1}{\delta} \right) \left(m^{-\frac{1}{2}}+  K(\eta)\, m^{\frac{1}{2}}\, \epsilon^2\right) \|f\|^2_{\rho} +   \tilde{C}m^{\frac{1}{2}} \, K(\eta)\, \log^{\frac{1}{2}}\left(\frac{1}{\delta}\right) E^2,
\end{align*}
with probability at least $1-8\delta$, where $K(\eta):= \frac{1+3\eta}{1-3\eta}$ and
\begin{align*}
    \epsilon = \frac{2}{\sqrt{N}} \left( 1+4\gamma\sigma d \sqrt{1+\sqrt{\frac{12}{d}\log\left(\frac{m}{\delta}\right)}} +\sqrt{\frac{1}{2}\log \left(\frac{1}{\delta}\right)}\right).
\end{align*}

\end{theorem}
The proof of Theorem \ref{Th: error for least square m<N} is given in Section \ref{proof generalization}. In this case, $K(\eta)$ is an upper bound for the condition number of the matrix $\frac{1}{N}\A\A^*$ using Theorem \ref{Th: Main Theorem (simplified)}.

\subsection{Overparameterized Regime, Sparse Regression}
Next, we consider the sparse regression setting, where we minimize the $\ell^1$ norm of the coefficients. The goal is to obtain a subset of the features which maintain an accurate approximation to the target function without needing to use the entire feature space, i.e. to obtain low complexity random feature models (see \cite{hashemi2021generalization}). This is implicitly connected to methods for pruning overparameterized networks and the \textit{lottery ticket hypothesis} \cite{frankle2018lottery}. Specifically, it is conjectured that there are smaller subnetworks of overparameterized randomly trained networks which are accurate representation of the target function with lower complexity. Our results also provide an additional motivation based on the structure of the transform of $f$, i.e. the decay of $\alpha(\boldsymbol{\omega})$.

We consider the $\ell^1$ basis pursuit denoising problem \cite{foucart2013invitation}
\begin{align}\label{BP problem}
    \bc^{\sharp} =\argmin{ \|\A\bc-\boldsymbol{y}\|_2 \leq \xi \sqrt{m} } \quad \|\bc\|_1,
\end{align}
where $\xi$ is a parameter related to the noise level (taking into account model inaccuracy as well). The constraint in \eqref{BP problem} is inexact compared to the interpolation condition used in \eqref{min-norm problem}, since one cannot guarantee the existence of exact sparse solutions of $\A\bc=\boldsymbol{y}$ in the presence of approximation error and measurement noise.

In this setting, we modify the approximation $f^{\sharp}$ to incorporate a pruning step as done in \cite{hashemi2021generalization}. Suppose that $\bc^{\sharp}$ is a solution of \eqref{BP problem}, we define $f^{\sharp}$ by 
\begin{align}\label{eq: truncated f sharp}
    f^{\sharp}(\bx) = \sum_{k\in \mathcal{S}^{\sharp}} c_k^{\sharp} \phi(\bx,\boldsymbol{\omega}_k),
\end{align}
where $\mathcal{S}^{\sharp}$ is the index set of the $s$-largest absolute entries of $\bc^{\sharp}$. This guarantees that the approximation only depends on at most $s<N$ random features. For any vector $\bc$, we denote by $\vartheta_{s,p}(\bc)$ the $\ell^p$-error of the best $s$-term approximation to $\bc$, i.e.,
$$
\vartheta_{s,p}(\bc) := \inf \{\|\bc-\Tilde{\bc}\|_p: \Tilde{\bc} \text{ is }s\text{ sparse}\}.
$$
This is a measure of the compressibility of a vector $\bc$ with respect to the $\ell^p$ norm and will appear in the risk bounds as a source of error from the sparsification.

To measure the conditioning of the system with respect to sparse recovery, we define the restricted isometry constant of a matrix $\A\in \mathbb{C}^{m\times N}$ using the notation from \cite{foucart2013invitation}. For an integer $s\leq N$, the $s$-th restricted isometry property (RIP) constant of $\A$, denoted by $\delta_s=\delta_s(\A)$,  is the smallest $\delta\geq 0$ such that
$$
(1-\delta)\|\bx\|_2^2 \leq \|\A\bx\|_2^2 \leq (1+\delta)\|\bx\|_2^2
$$
holds for all $s$-sparse $\bx$.  We have the following estimate of the $s$-RIP constant of the random feature matrix $\A$.

\begin{theorem}[\textbf{Estimate on Restricted Isometry Constants}]\label{thm: RIP estimate}
Let the data $\{\bx_j\}_{j\in [m]}$, weights $\{\boldsymbol{\omega}_k\}_{k\in [N]}$, and random feature matrix $\A\in\mathbb{C}^{m\times N}$ satisfy Assumption \ref{assumption}. For $\eta_1, \eta_2, \delta \in (0,1)$ and some integer $s\geq 1$, if
\begin{align*}
    m&\geq C_1 \eta_1^{-2}\, s\, \log(\delta^{-1}) \\
    \frac{m}{\log(3m)} &\geq C_2 \eta_2^{-2}\, s \,\log^2(s) \log\left(\frac{N}{9\log(2m)+3}\right) \\
    \sqrt{\delta}&\,\eta_1\, (4\gamma^2\sigma^2+1)^{\frac{d}{4}}\geq N.
\end{align*}
where $C_1$ and $C_2$ are universal positive constants, then with probability at least $1-2\delta$, the $s$-RIP constant $\delta_{s}\left(\frac{1}{\sqrt{m}}A\right)$ is bounded by $f(\eta_1,\eta_2)$, where
$$
f(\eta_1,\eta_2) := 3\eta_1 + \eta_2^2 + \sqrt{2}\eta_2
$$
\end{theorem}

With the above theorem, we can derive the following result.

\begin{theorem}[\textbf{Sparse Regression Risk Bounds}]\label{Th: error for bp}
Let the data $\{\bx_j\}_{j\in [m]}$, weights $\{\boldsymbol{\omega}_k\}_{k\in [N]}$, and random feature matrix $\A\in\mathbb{C}^{m\times N}$ satisfy Assumption \ref{assumption} and let the noise satisfy Assumption \ref{assumption2}.  Let $c^\sharp$ be obtained from  \eqref{BP problem} with $\xi = \sqrt{2(\epsilon^2 \|f\|_{\rho}+E^2)}$. If the following conditions hold
\begin{align*}
    &\frac{m}{\log(3m)} \geq C s \log^2(2s)\log\left(\frac{N}{9\log(2m)}+3\right) \\
    & \frac{\sqrt{\delta}}{10}\left(4\gamma^2 \sigma^2+1\right)^{\frac{d}{4}} \geq N\\
    &\delta \geq N^{-\log^2(2s)\log(3m)},
\end{align*}
then with probability at least $1-8\delta$ the following risk bound holds: 
\begin{align*}
    R(f^{\sharp}) \leq C'\,\left(1+Nm^{-\frac{1}{2}}\log^{\frac{1}{2}}\left(\frac{1}{\delta}\right)\right) \, (\epsilon^2 \|f\|^2_{\rho}+E^2) + C''\,\left(1+Nm^{-\frac{1}{2}}s^{-1}\log^{\frac{1}{2}}\left(\frac{1}{\delta}\right)\right)\,\vartheta_{s,1}(\bc^*)^2,
\end{align*}
where $f^{\sharp}$ is defined as \eqref{eq: truncated f sharp}, $\bc^\star$ is the vector
whose components are defined by $\bc^\star_k = \frac{1}{N} \frac{\alpha(\boldsymbol{\omega}_k)}{\rho(\boldsymbol{\omega}_k)},$ for $k\in[N]$
and
\begin{align*}
    \epsilon = \frac{2}{\sqrt{N}} \left( 1+4\gamma\sigma d \sqrt{1+\sqrt{\frac{12}{d}\log\left(\frac{m}{\delta}\right)}} +\sqrt{\frac{1}{2}\log \left(\frac{1}{\delta}\right)}\right).
\end{align*}
The constants $C,C',C''>0$ are universal constants independent of the dimension $d$.
\end{theorem}
The proofs are given in Section \ref{proof generalization}. Note that in the worst case setting:
$$\vartheta_{s,1}(\bc^*)  \leq \left(1-\frac{s}{N}\right) \, \|f\|_\rho.$$ 

\begin{remark}
The sparsity parameter $s$ in Theorem \ref{Th: error for bp} can be defined by the user since it is obtained by the pruning step used in \eqref{eq: truncated f sharp}. When $s=N$, $\vartheta_{N,1}(\bc^*)=0$  and thus Theorem \ref{Th: error for least square m<N}
and Theorem \ref{Th: error for bp} generally agree in terms of the risk's dependency on  $\epsilon^2 \|f\|^2_{\rho}$ and $E^2$. On the other hand, if $\alpha(\boldsymbol{\omega})$ decays much faster than $\rho(\boldsymbol{\omega})$ or has small support within the support set of $\rho(\boldsymbol{\omega})$, then the vector $\bc^\star$ will be compressible and the sparse regression problem would be better. Lastly, the result in Theorem \ref{Th: error for bp} is more robust to noise than in the $\ell^2$ setting.
\end{remark}

\newpage
\section{Conditioning in Three Regimes}\label{proof of condition}
In this section, we prove Theorem \ref{Th: Main Theorem (simplified)}. The proofs for Theorem \ref{Th: Main Theorem (simplified)} Part (a) and (b) depend on the matrix Bernstein inequality \ref{Matrix Bernstein Ineq}, and Part (c) follows from a more direct estimate on the expectation of the eigenvalues of the Gram matrix.

\begin{proof}[Proof of Theorem \ref{Th: Main Theorem (simplified)}]

To prove Part (a), we bound each of the following three terms
\begin{align}\label{eq: three terms}
    \left|\lambda_{k}\left(\frac{1}{m}\A^*\A\right) - 1\right| &\leq \left\|\frac{1}{m}\A^*\A - \boldsymbol{I}_N\right\|_2 \notag \\
    &\leq \frac{1}{m}\left\|\sum_{\ell=1}^m(\X_\ell\X_\ell^* - \mathbb{E}_{\bx}\X_\ell\X_\ell^*)\right\|_2 + \frac{1}{m}\left\|\sum_{\ell=1}^m(\mathbb{E}_{\bx}\X_\ell\X_\ell^* - \mathbb{E}_{\bx,\bomega}\X_\ell\X_\ell^*)\right\|_2 + \left\|\boldsymbol{L}\right\|_2, 
\end{align}
 by $\eta>0$, where $\boldsymbol{L} := \mathbb{E}_{\bx,\bomega}(\X_\ell \X_{\ell}^*) - \boldsymbol{I}_N$.

First, $\boldsymbol{L}$ is a symmetric matrix with $L_{j,j}=0$ and $$L_{j,k} = \mathbb{E}_{\bx,\bomega} [\exp(i\langle \bx_\ell,\bomega_k-\bomega_j \rangle)] = (2\gamma^2\sigma^2+1)^{-\frac{d}{2}},$$ for $j,k\in [N], j\neq k$. Let $\boldsymbol{z}$ be a unit vector, then by H\"{o}lder's inequality
\begin{align*}
    \left|\langle \boldsymbol{L}\boldsymbol{z},\boldsymbol{z} \rangle\right| = \left|\sum_{\substack{j,k=1\\ j\neq k}}^N \left(\frac{1}{2\gamma^2\sigma^2+1}\right)^{\frac{d}{2}} \, z_j\, \overline{z_k}\right|  &\leq \left(\frac{1}{2\gamma^2\sigma^2+1}\right)^{\frac{d}{2}} \,\|\boldsymbol{z}\|_1^2 \notag\\
    &\leq  N\, \left(\frac{1}{2\gamma^2\sigma^2+1}\right)^{\frac{d}{2}} \notag \\
    &\leq N\, \left(\frac{1}{4\gamma^2\sigma^2+1}\right)^{\frac{d}{4}} \notag\\
    &\leq \sqrt{\delta}\eta,
\end{align*}
where assumption \eqref{ineq: simplified uncertainty principle} was used in the last inequality. Therefore, the third term in \eqref{eq: three terms} is bounded by $\|\boldsymbol{L}\|_2\leq\sqrt{\delta}\eta \leq \eta$.

Next, define the random variable $Z$ by
$$
Z(\bomega_1,\dots,\bomega_N) := \|\mathbb{E}_{\bx}(\X_{\ell}\X_{\ell}^*) - \mathbb{E}_{\bx,\bomega}(\X_{\ell}\X_{
\ell}^*)\|_2.
$$
Since $\mathbb{E}_{\bx}(\X_{\ell}\X_{\ell}^*)$ depends only on $\{\bomega_k\}_{k\in [N]}$, the random variable $Z$ is independent of $\{\bx_j\}_{j\in [m]}$. Note that $Z$ is the norm of a dependent random matrix, and thus we cannot guarantee an exponential concentration rate and instead we apply Markov's inequality. The second moment of $Z$ is bounded by
\begin{align*}
    \mathbb{E}_{\bomega}(Z(\bomega_1,\dots,\bomega_N)^2) &= \mathbb{E}_{\bomega} \|\mathbb{E}_{\bx}(\X_{\ell}\X_{\ell}^*) - \mathbb{E}_{\bx,\bomega}(\X_{\ell}\X_{\ell}^*)\|_2^2 \notag \\
    &\leq \mathbb{E}_{\bomega} \|\mathbb{E}_{\bx}(\X_{\ell}\X_{\ell}^*) - \mathbb{E}_{\bx,\bomega}(\X_{\ell}\X_{\ell}^*)\|_F^2 \notag \\
    &= \sum_{\substack{j,k=1\\ j\neq k}} \mathbb{E}_{\bomega} \left|\exp\left(-\frac{\gamma^2}{2}\|\bomega_k-\bomega_j\|_2^2\right) - \left(\frac{1}{2\gamma^2\sigma^2+1}\right)^{\frac{d}{2}}\right|^2 \notag \\
       &= \sum_{\substack{j,k=1\\ j\neq k}}\left( \mathbb{E}_{\bomega} \exp\left(-\gamma^2\|\bomega_k-\bomega_j\|_2^2\right) - \left(\frac{1}{2\gamma^2\sigma^2+1}\right)^{d}\right) \notag \\
    &\leq N^2 \left(\frac{1}{4\gamma^2\sigma^2+1}\right)^{\frac{d}{2}}.
\end{align*}
By Markov's inequality and the complexity assumption \eqref{ineq: simplified uncertainty principle},
\begin{align}
    \mathbb{P}_{\bomega}(Z(\bomega) \geq \eta) &\leq \mathbb{P}_{\bomega}\left(Z(\bomega)\geq \frac{N}{\sqrt{\delta}}\left(\frac{1}{4\gamma^2\sigma^2+1}\right)^{\frac{d}{4}}\right) \notag \\
    &\leq \frac{\delta (4\gamma^2\sigma^2+1)^{\frac{d}{2}}}{N^2}\, \mathbb{E}_{\bomega} (Z(\bomega)^2) \notag\\
    &\leq \delta.
\end{align}
This implies that with probability at least $1-\delta$ with respect to the draw of $\{\bomega_k\}_{k\in [N]}$, we have $\|\mathbb{E}_{\bx}(\X_{\ell}\X_{\ell}^*) - \mathbb{E}_{\bx,\bomega}(\X_{\ell}\X_{\ell}^*)\|_2 \leq \eta$.

Next, we condition the remaining term in \ref{eq: three terms} on the draw of $\{\bomega_k\}_{k\in [N]}$ where we have shown that $\|\mathbb{E}_{\bx}(\X_{\ell}\X_{\ell}^*) - \mathbb{E}_{\bx,\bomega}(\X_{\ell}\X_{\ell}^*)\|_2 \leq \eta$. Specifically, define the random matrices $\{\boldsymbol{Y}_\ell\}_{\ell \in [m]}$ by
$$
\boldsymbol{Y}_{\ell} =  \X_{\ell}\X_{\ell}^* - \mathbb{E}_{\bx}(\X_{\ell}\X_{\ell}^*),
$$
where each $\boldsymbol{Y}_{\ell}$ depends only on $\{\bx_\ell\}_{\ell\in [m]}$, conditioned on the given draw of $\{\bomega_k\}_{k\in [N]}$. The matrix $\boldsymbol{Y}_\ell$ satisfies $(\boldsymbol{Y}_{\ell})_{j,j} = 0$ and $(\boldsymbol{Y}_{\ell})_{j,k} = \exp(i\langle\bx_\ell, \bomega_k-\bomega_j \rangle) - \exp(-\gamma^2\|\bomega_k-\bomega_j\|_2^2/2)$ for $j,k\in [N], j\neq k$. The norms are bounded by
$$
\|\boldsymbol{Y}_{
\ell}\|_2 \leq \max_{j\in[N]} \, \sum_{\substack{k=1\\ k\neq j}}^N |e^{i\langle\bx_\ell, \bomega_k-\bomega_j \rangle} - e^{-\frac{\gamma^2}{2}\|\bomega_k-\bomega_j\|_2^2}| \leq 2N \quad \ell\in [m],
$$
using Gershgorin's theorem. By the previous results, we have
\begin{align*}
\|\mathbb{E}_{\bx}\X_{\ell}\X_{\ell}^*\|_2 &\leq \|\mathbb{E}_{\bx}\X_{\ell}\X_{\ell}^* - \mathbb{E}_{\bx,\bomega}\X_{\ell}\X_{\ell}^*\|_2 + \|\mathbb{E}_{\bx,\bomega}\X_{\ell}\X_{\ell}^*\|_2 \\
&\leq \eta + \|\boldsymbol{I}_N + \boldsymbol{L}\|_2 \\
&\leq 2\eta+1
\end{align*}
and thus
\begin{align*}
    \left\|\sum_{\ell=1}^m \mathbb{E}_{\bx}(\boldsymbol{Y}_{\ell}^2)\right\|_2 &\leq \sum_{\ell=1}^m \left\|\mathbb{E}_{\bx}(\boldsymbol{Y}_{\ell}^2)\right\|_2 \notag \\
    &\leq \sum_{\ell=1}^m \left\|N\mathbb{E}_{\bx}(\X_{\ell}\X_{\ell}^*) - (\mathbb{E}_{\bx}\X_{\ell}\X_{\ell}^*)^2\right\|_2 \notag \\
    &\leq \sum_{\ell=1}^m \left( N\|\mathbb{E}_{\bx}(\X_\ell\X_{\ell}^*)\|_2 + \|\mathbb{E}_{\bx}(\X_\ell\X_{\ell}^*)\|_2^2 \right) \notag \\
    &\leq m[N(1+2\eta)+(1+2\eta)^2].
\end{align*}
noting that $\X_{\ell}^*\X_{\ell}=N$. Applying Theorem \ref{Matrix Bernstein Ineq} (by setting the parameters in the theorem's statement to $\sigma^2=m\left(\frac{5}{3}N+\frac{25}{9}\right)$ and $K=2N$ and assuming that $\eta\leq \frac{1}{3}$) yields
\begin{align}
    \mathbb{P}_{\bx}\left( \frac{1}{m}\|\sum_{\ell=1}^m (\X_\ell\X_{\ell}^* - \mathbb{E}_{\bx}\X_{\ell}\X_{\ell}^*)\|_2\geq \eta\right) & = \mathbb{P}_{\bx}\left(\|\sum_{\ell=1}^m \boldsymbol{Y}_\ell\|_2 \geq m\eta \right) \notag \\
    &\leq 2N\exp\left( -\frac{m\eta^2}{\frac{10}{3}N+\frac{50}{9}+\frac{4N\eta}{3}}\right).
\end{align}
If $m\geq 4 N \eta^{-2} \log\left(\frac{2N}{\delta}\right)$ (assuming $N>25$), then with probability at least $1-\delta$ with respect to the draw of $\{\bx_j\}_{j\in [m]}$, we have $m^{-1}\|\sum_{\ell=1}^m\X_{\ell}\X_{\ell}^* - \sum_{\ell=1}^m\mathbb{E}_{\bx}\X_{\ell}\X_{\ell}^*\|_2 \leq \eta$.

Altogether,  \eqref{eq: three terms} is bounded by
\begin{align*}
    \frac{1}{m}\|\sum_{\ell=1}^m(\X_\ell\X_\ell^* - \mathbb{E}_{\bx}\X_\ell\X_\ell^*)\|_2 + \frac{1}{m}\|\sum_{\ell=1}^m(\mathbb{E}_{\bx}\X_\ell\X_\ell^* - \mathbb{E}_{\bx,\bomega}\X_\ell\X_\ell^*)\|_2 + \|\boldsymbol{L}\|_2
    &\leq 3\eta
\end{align*}
with probability at least $(1-\delta)^2\geq 1-2\delta$ if the conditions in the previous steps are satisfied.

For Part (b), note that the samples $\left\{\bx_j\right\}_{j\in[m]}$ and the weights $\left\{\boldsymbol{\omega}_k\right\}_{k\in[N]}$ take real values and their distributions are symmetric about $\boldsymbol{0}$. Therefore, the feature matrix $\A=[a_{j,k}]$ where $a_{j,k}:=\exp(i\langle \bx_j, \boldsymbol{\omega}_k \rangle)$ and its conjugate transpose $\A^*=[\overline{a}_{k,j}]$ where $\overline{a}_{k,j}=\exp(-i \langle \boldsymbol{\omega}_k, \bx_j \rangle)$ have the same distribution. Arguing similarly as in the proof of Part (a) proves Part (b) as well. Essentially, one can consider the system where the meaning of $\boldsymbol{\omega}$ and $\bx$ are switched.

For Part (c), consider the case when $m=N$. The matrix $\frac{1}{N}\A^*\A$ can be written as the following rank-1 decomposition
$$
\frac{1}{N}\A^*\A = \frac{1}{N}\sum_{\ell=1}^N \X_\ell \X_\ell^*,
$$
where $\X_\ell$ is the $\ell$-th column of $\A^*$. Since the Gram matrix, $\frac{1}{N}\A^*\A$, is Hermitian we can use the Rayleigh quotient to bound the maximum eigenvalue from below by
$$
\lambda_{\max}\left( \frac{1}{N} \A^*\A \right) \geq \frac{1}{N} \langle \A^*\A\boldsymbol{v}, \boldsymbol{v} \rangle \quad \text{for all } \boldsymbol{v}\in\mathbb{C}^N \text{ with } \|\boldsymbol{v}\|_2 =1.
$$
Thus setting $\boldsymbol{v} = \frac{1}{\sqrt{N}}\X_1$ yields
\begin{align*}
    \lambda_{\max}\left( \frac{1}{N} \A^*\A \right) &\geq \frac{1}{N^2} \sum_{\ell=1}^N \X_1^*\X_\ell\X_\ell^*\X_1  \notag \\
    &=\frac{1}{N^2}\left( N^2 + \sum_{
    \ell=2}^N \X_1^*\X_\ell\X_\ell^*\X_1 \right) \notag \\
    & = 1 + \frac{1}{N^2} \sum_{\ell=2}^N \sum_{j,k=1}^N \exp(i\langle \bx_1 - \bx_\ell, \boldsymbol{\omega}_j-\boldsymbol{\omega}_k \rangle) \notag \\
     & = 1 + \frac{(N-1)N}{N^2}+ \frac{1}{N^2}\sum_{\ell=2}^N\, \sum_{\substack{j,k=1\\ j\neq k}}^N \exp(i\langle \bx_1 - \bx_\ell, \boldsymbol{\omega}_j-\boldsymbol{\omega}_k \rangle)\notag.
\end{align*}
Taking the expectation on both sides yields
\begin{align*}
    \mathbb{E}\lambda_{\max}\left( \frac{1}{N}\A^*\A \right)\geq 2 - \frac{1}{N} + \frac{(N-1)^2}{N}\left( \frac{1}{\sqrt{4\gamma^2\sigma^2+1}} \right)^d \geq 2-\frac{1}{N},
\end{align*}
where we used the characteristic function of the normal distribution.

Similarly, for the smallest eigenvalue, we have
$$
\lambda_{\min}\left( \frac{1}{N} \A^*\A \right) \leq \frac{1}{N} \langle \A^*\A\boldsymbol{v}, \boldsymbol{v} \rangle \quad \text{for all } \boldsymbol{v}\in\mathbb{C}^N \text{ with } \|\boldsymbol{v}\|_2 =1.
$$
Since $\left\{\X_\ell\right\}_{\ell\in [N]}$ are vectors in $\mathbb{C}^N$, we can find a unit vector $\boldsymbol{u}=[u_1,\dots,u_N]^T \in \mathbb{C}^N$ such that $\boldsymbol{u}$ is orthogonal to  $\text{span}\left(\left\{\X_\ell\right\}_{\ell\in [N-1]}\right)$ and thus
\begin{align*}
    \lambda_{\min} \left( \frac{1}{N}\A^*\A \right) &\leq \frac{1}{N}\langle \A^*\A\boldsymbol{u}, \boldsymbol{u} \rangle \notag \\
    &\leq \frac{1}{N} \boldsymbol{u}^*\X_N\X_N^*\boldsymbol{u} \notag \\
    & = \frac{1}{N}\sum_{j,k=1}^N \overline{u_j}u_k \exp(i\langle \bx_N, \boldsymbol{\omega}_k - \boldsymbol{\omega}_j \rangle).
\end{align*}
The vector $\boldsymbol{u}$ depends linearly on $\left\{\X_\ell\right\}_{\ell\in [N-1]}$, thus the components $u_j$ for $j\in[N]$ are random variables depending on $\{\bx_j\}_{j\in [N-1]}$ and weights $\{\boldsymbol{\omega}_k\}_{\ell\in [N]}$ but are independent of $\bx_N$. By taking the expectation on both sides and applying Fubini's theorem, we bound the expectation on the minimum eigenvalue by
\begin{align*}
    \mathbb{E} \lambda_{\min}\left( \frac{1}{N}\A^*\A \right) &\leq 
    \frac{1}{N} + \frac{1}{N}\,  \mathbb{E}\sum_{j\neq k} \overline{u_j}u_k \exp(i\langle \bx_N, \boldsymbol{\omega}_k - \boldsymbol{\omega}_j \rangle) \notag \\
    &=\frac{1}{N} +  \frac{1}{N} \, \mathbb{E}_{\boldsymbol{\omega}_1,\dots,\boldsymbol{\omega}_N,\bx_1,\dots,\bx_{N-1}}\sum_{j\neq k} \overline{u_j}u_k \mathbb{E}_{\bx_N}[\exp(i\langle \bx_N,\boldsymbol{\omega}_k-\boldsymbol{\omega}_j \rangle)]\\
    &=\frac{1}{N} +  \frac{1}{N} \, \mathbb{E}_{\boldsymbol{\omega}_1,\dots,\boldsymbol{\omega}_N,\bx_1,\dots,\bx_{N-1}}\sum_{j\neq k} \overline{u_j}u_k \exp\left(-\frac{\gamma^2}{2}\|\boldsymbol{\omega}_k-\boldsymbol{\omega}_j\|_2^2\right) \notag \\
    &\leq \frac{1}{N} + \frac{1}{N}\mathbb{E}_{\boldsymbol{\omega}_1,\dots,\boldsymbol{\omega}_N} \sqrt{\sum_{j\neq k}\exp\left( -\gamma^2\|\boldsymbol{\omega}_k-\boldsymbol{\omega}_j\|_2^2 \right)} \notag \\
    &\leq \frac{1}{N} + \frac{1}{N}\sqrt{\sum_{j\neq k}\mathbb{E}_{\boldsymbol{\omega}_1,\dots,\boldsymbol{\omega}_N}\exp\left( -\gamma^2\|\boldsymbol{\omega}_k-\boldsymbol{\omega}_j\|_2^2 \right)} \notag \\
    &\leq \frac{1}{N} + \left(1-\frac{1}{N}\right)^{\frac{1}{2}}\left(\frac{1}{4\gamma^2\sigma^2+1}\right)^{\frac{d}{4}},
\end{align*}
where in the fourth line we use the Cauchy-Schwarz inequality to bound
\begin{align*}
\sum_{j\neq k} \overline{u_j}u_k \exp\left(-\frac{\gamma^2}{2}\|\boldsymbol{\omega}_k-\boldsymbol{\omega}_j\|_2^2\right) &\leq \sqrt{\sum_{j\neq k} |u_j|^2|u_k|^2} \sqrt{\sum_{j\neq k}\exp(-\gamma^2\|\boldsymbol{\omega}_k-\boldsymbol{\omega}_j\|_2^2)} \notag \\
&\leq \sqrt{\|u\|_2^4} \sqrt{\sum_{j\neq k}\exp(-\gamma^2\|\boldsymbol{\omega}_k-\boldsymbol{\omega}_j\|_2^2)} \notag \\
&\leq \sqrt{\sum_{j\neq k}\exp(-\gamma^2\|\boldsymbol{\omega}_k-\boldsymbol{\omega}_j\|_2^2)},
\end{align*}
and in the fifth line we use the Jensen's inequality. This completes the proof.
\end{proof}

\section{Proofs of Generalization Theorems from Section \ref{sec:gen}} \label{proof generalization}

The proofs of Theorem \ref{Th: error for least square m>N}, \ref{Th: error for least square m<N} and \ref{Th: error for bp} follow a similar structure in which we show that each trained vector $\bc^\sharp$ is close to the best $\phi$ approximation $\bc^\star$ from \cite{rahimi2008uniform,rahimi2008weighted}. To related the coefficient vectors to the risk, we utilize the conditioning result, Theorem \ref{Th: Main Theorem (simplified)}, over a finite set of samples to approximate the risk by a new discrete system, as done in Theorem 1 from \cite{hashemi2021generalization}. Since the proofs of Theorem \ref{Th: error for least square m>N} and \ref{Th: error for least square m<N} are similar, we present them together.
\begin{proof}[Proof of Theorem \ref{Th: error for least square m>N} and \ref{Th: error for least square m<N}]
We will work directly with the $L^2$ norm, which can be decomposed into two parts
$$
\|f-f^{\sharp}\|_{L^2(d\mu)}\leq \|f-f^*\|_{L^2(d\mu)} + \|f^*-f^{\sharp}\|_{L^2(d\mu)}
$$
by triangle inequality. The  approximation $f^\sharp$ is defined in \eqref{eq: def of f-sharp} and the best $\phi$ based approximation $f^\star$ is defined by
\begin{equation}\label{eq: def of f-star}
    f^{\star}(\X) = \sum_{k=1}^N c^{\star}_k\, \phi(\bx,\boldsymbol{\omega}_k) = \frac{1}{N} \sum_{k=1}^N \frac{\alpha(\boldsymbol{\omega}_k)}{\rho(\boldsymbol{\omega}_k)}\, \phi(\bx,\boldsymbol{\omega}_k)
\end{equation}
with $c^{\star}_k=N^{-1} \, \frac{\alpha(\boldsymbol{\omega}_k)}{\rho(\boldsymbol{\omega}_k)}$ for all $k\in[N]$. If the condition on $N$ in Lemma \ref{lm: L2 error of f-f_star} is satisfied, then with probability at least $1-\delta$
\begin{align*}
    \|f-f^*\|_{L^2(d\mu)}\leq \epsilon \|f\|_{\rho}.
\end{align*}

To bound $\|f^*-f^{\sharp}\|_{L^2(d\mu)}$, we use McDiarmid's inequality and argue similarly as in Lemma 2 from \cite{hashemi2021generalization}. Let $\{\boldsymbol{z}_j\}_{j\in[m]}$ be i.i.d. random variables sampled from the distribution $\mu$ and independent from the $\{\bx_j\}_{j\in[m]}$ and $\{\boldsymbol{\omega}_k\}_{k\in[N]}$. Thus $\{\boldsymbol{z}_j\}_{j\in[m]}$ is also independent of $\bc^{\sharp}$ and $\bc^\star$. We define the random variable
\begin{align*}
    v(\boldsymbol{z}_1,\dots,\boldsymbol{z}_m) := \|f^*-f^{\sharp}\|_{L^2(d\mu)}^2 - \frac{1}{m}\sum_{j=1}^m |f^*(\boldsymbol{z}_j)-f^{\sharp}(\boldsymbol{z}_j)|^2.
\end{align*}
Note that for any $j\in[m]$,
$$
\mathbb{E}_{\boldsymbol{z}}\left[|f^*(\boldsymbol{z}_j) - f^{\sharp}(\boldsymbol{z}_j)|^2\right]=\mathbb{E}_{\boldsymbol{z_1}, \dots,\boldsymbol{z_m}}\left[|f^*(\boldsymbol{z}_j) - f^{\sharp}(\boldsymbol{z}_j)|^2\right] = \|f^*-f^{\sharp}\|_{L^2(d\mu)}^2,
$$
thus $\mathbb{E}_{\boldsymbol{z}}[v] = 0$. Perturbing the $k$-th component of $v$ yields
\begin{align*}
    |v(\boldsymbol{z}_1,\dots,\boldsymbol{z}_k,\dots,\boldsymbol{z}_m) - v(\boldsymbol{z}_1,\dots,\Tilde{\boldsymbol{z}}_k,\dots,\boldsymbol{z}_m)| \leq \frac{1}{m}\left\vert |f^*(\boldsymbol{z}_k)-f^{\sharp}(
    \boldsymbol{z}_k)|^2 - |f^*(\Tilde{\boldsymbol{z}}_k) - f^{
    \sharp}(\Tilde{\boldsymbol{z}}_k)|^2 \right\vert.
\end{align*}
By Cauchy-Schwarz inequality, for any $\boldsymbol{z}$ we have
\begin{align*}
    |f^*(\boldsymbol{z})-f^{\sharp}(\boldsymbol{z})|^2 = \left\vert \sum_{k=1}^N (c^*_k-c^{\sharp}_k) \phi(
    \boldsymbol{z},\boldsymbol{\omega}_k) \right\vert^2\leq N \|\bc^*-\bc^{\sharp}\|_2^2,
\end{align*}
which holds since $\left|\phi(\boldsymbol{z},\boldsymbol{\omega})\right|= 1$. Thus the difference is bounded by
\begin{align}\label{eq: definition of Delta_v}
    |v(\boldsymbol{z}_1,\dots,\boldsymbol{z}_k,\dots,\boldsymbol{z}_m) - v(\boldsymbol{z}_1,\dots,\Tilde{\boldsymbol{z}}_k,\dots,\boldsymbol{z}_m)| &\leq \frac{2N}{m} \|\bc^*-\bc^{\sharp}\|_2^2 =: \Delta.
\end{align}
Next, we apply McDiarmid's inequality $\mathbb{P}_{\boldsymbol{z}}(v-\mathbb{E}_{\boldsymbol{z}}[v]\geq t)\leq \exp(-\frac{2t^2}{m\Delta^2})$, by setting
\begin{align*}
    t = \Delta \sqrt{\frac{m}{2}\log\left(\frac{1}{\delta}\right)}
\end{align*}
which implies that
\begin{align*}
    \|f^*-f^{\sharp}\|_{L^2(d\mu)}^2 &\leq \frac{1}{m} \sum_{j=1}^m |f^*(\boldsymbol{z}_j)-f^{\sharp}(\boldsymbol{z}_j)|^2 + N\sqrt{\frac{2}{m}\log\left(\frac{1}{\delta}\right)}\, \|\bc^{\star}-\bc^{\sharp}\|_2^2, \notag \\
    &= \frac{1}{m} \|\boldsymbol{\tilde{A}}(\bc^{\star}-\bc^{\sharp})\|_2^2 + N\sqrt{\frac{2}{m}\log\left(\frac{1}{\delta}\right)}\, \|\bc^{\star}-\bc^{\sharp}\|_2^2
\end{align*}
with probability at least $1-\delta$ with respect to draw of $\{\boldsymbol{z}_j\}_{j\in[m]}$, and $\boldsymbol{\Tilde{A}}\in\mathbb{C}^{m\times N}$ is the random feature matrix with $\Tilde{a}_{j,k} = \exp(i\langle \boldsymbol{z}_j,\boldsymbol{\omega}_k \rangle)$.

In the underparameterized regime $m>N$, we apply Part (a) of Theorem \ref{Th: Main Theorem (simplified)} to $\boldsymbol{\tilde{A}}$ and we obtain
\begin{align*}
    \|f^*-f^{\sharp}\|_{L^2(d\mu)}^2 &\leq \frac{1}{m}\|\Tilde{\A}(\bc^*-\bc^{\sharp})\|_2^2 + N\sqrt{\frac{2}{m}\log\left(\frac{1}{\delta}\right)}\|\bc^*-\bc^{\sharp}\|_2^2 \notag\\
   & =\left\| \frac{1}{\sqrt{m}}\Tilde{\A}(\bc^*-\bc^{\sharp})\right\|_2^2 + N\sqrt{\frac{2}{m}\log\left(\frac{1}{\delta}\right)}\|\bc^*-\bc^{\sharp}\|_2^2 \notag \\
   & \leq \left(1+3\eta \right)\left\|\bc^*-\bc^{\sharp}\right\|_2^2 + N\sqrt{\frac{2}{m}\log\left(\frac{1}{\delta}\right)}\|\bc^*-\bc^{\sharp}\|_2^2 \notag \\
    &\leq \left(1+3\eta+N\sqrt{\frac{2}{m}\log\left(\frac{1}{\delta}\right)}\right) \|\bc^*-\bc^{\sharp}\|_2^2
\end{align*}
with probability at least $1-3\delta$. 

To estimate $\|\bc^\sharp - \bc^*\|_2$, we utilize the the pseudo-inverse $A^\dagger$ and define the residual $\bh = \A\bc^* - \boldsymbol{y}$, so that
\begin{align*}
    \|\bc^\sharp - \bc^*\|_2 = \|\A^{\dagger}\boldsymbol{y} - \bc^*\|_2 & = \|\A^\dagger (\A\bc^*-\bh) - \bc^*\|_2 \notag \\
    &\leq \|(\A^\dagger\A-\boldsymbol{I})\bc^*\|_2 + \|\A^\dagger\|_2\,\|\bh\|_2.
\end{align*}
Using H\"{o}lder's inequality and Assumption \ref{assumption2} on $\|\boldsymbol{e}\|_2$, the residual is bounded by
\begin{align}
    \|\bh\|_2 = \sqrt{\sum_{j=1}^m \left( f^*(\bx_j) - f(\bx_j) - e_j \right)^2}
    &\leq \sqrt{\sum_{j=1}^m \left( f^*(\bx_j) - f(\bx_j) \right)^2} + \|\boldsymbol{e}\|_2 \notag \\
    &\leq \sqrt{m} \sup_{\|\X\|_2\leq R} |f^*(\bx)-f(\bx)| + \sqrt{m} E \notag\\
    &\leq \sqrt{m}(\epsilon\|f\|_{\rho} + E) \label{eq: xi error}
\end{align}
with probability $1-2\delta$ if conditions in Lemma \ref{lm: samples lie within a ball} and Lemma \ref{lm: l_infty error of f-f_star} are satisfied. Therefore, the bound becomes
\begin{align*}
    \|\bc^\sharp - \bc^*\|_2 
    &\leq \|(\A^\dagger\A-\boldsymbol{I})\bc^*\|_2 + \sqrt{m}\, \|\A^\dagger\|_2\,\left(\epsilon\|f\|_{\rho} + E \right).\notag
\end{align*}

Note that when $m>N$, the first term is $\|(\A^\dagger\A-\boldsymbol{I})\bc^*\|_2=0$ since $\A^\dagger\A= (\A^*\A)^{-1}\A^*\A = \boldsymbol{I}$. Using Theorem \ref{Th: Main Theorem (simplified)}, the operator norm of the pseudo-inverse is bounded by
\begin{align*}
    \|\A^\dagger\|_2 = \frac{1}{\sqrt{\lambda_{\text{min}}(\A^*\A)}}= \frac{1}{\sqrt{m}}\, \frac{1}{\sqrt{\lambda_{\text{min}}(m^{-1}\A^*\A)}}\leq \frac{1}{\sqrt{m}}{\frac{1}{\sqrt{1-3\eta}}}
 \end{align*}
with probability at least $1-2\delta$ and thus
\begin{align*}
    \|f-f^{\sharp}\|_{L^2(d\mu)} &\leq \|f-f^*\|_{L^2(d\mu)} + \|f^*-f^{\sharp}\|_{L^2(d\mu)} \notag \\
    &\leq \epsilon\|f\|_{\rho} + \|f^*-f^{\sharp}\|_{L^2(d\mu)} \notag \\
    &\leq \epsilon\|f\|_{
    \rho} +\left( 1+3\eta + N\sqrt{\frac{2}{m}\log\left(\frac{1}{\delta}\right)} \right)^{\frac{1}{2}} \|\bc^{\star}-\bc^{\sharp}\|_2 \notag \\
    &\leq 2\sqrt{K(\eta)}\left(1+N^{\frac{1}{2}}m^{-\frac{1}{4}}\log^{\frac{1}{4}}\left(\frac{1}{\delta}\right)\right) (\epsilon\|f\|_{\rho}+E)
\end{align*}
with probability at least $1-8\delta$.

In the overparameterized regime $m<N$, we apply Part (b) of Theorem \ref{Th: Main Theorem (simplified)} to $\boldsymbol{\tilde{A}}$ and the error of $\|f-f^{\sharp}\|_{L^2(d\mu)}^2$ is bounded by
\begin{align*}
    \|f^*-f^{\sharp}\|_{L^2(d\mu)}^2 &\leq \frac{1}{m} \|\boldsymbol{\tilde{A}}(\bc^{\star}-\bc^{\sharp})\|_2^2 + N\sqrt{\frac{2}{m}\log\left(\frac{1}{\delta}\right)}\, \|\bc^{\star}-\bc^{\sharp}\|_2^2 \notag \\
    & = \frac{N}{m}\|\frac{1}{\sqrt{N}}\boldsymbol{\tilde{A}}(\bc^{\star}-\bc^{\sharp})\|_2^2 + N\sqrt{\frac{2}{m}\log\left(\frac{1}{\delta}\right)}\, \|\bc^{\star}-\bc^{\sharp}\|_2^2 \notag \\
    &\leq \frac{N}{m} \left(1+\delta_N\left(\frac{1}{\sqrt{N}}\boldsymbol{\tilde{A}}\right) \right)\|\bc^{\star}-\bc^{\sharp}\|_2^2 + N\sqrt{\frac{2}{m}\log\left(\frac{1}{\delta}\right)} \|\bc^{\star}-\bc^{\sharp}\|_2^2\\
    &\leq \left(\frac{N}{m}\left(1+3\eta \right) + N\sqrt{\frac{2}{m}\log\left(\frac{1}{\delta}\right)}\right) \|\bc^{\star}-\bc^{\sharp}\|_2^2
\end{align*}
with probability at least $1-3\delta$. For $\|\bc^{\star}-\bc^{\sharp}\|_2$, we still have
$$
\|\bc^{\star}-\bc^{\sharp}\|_2 \leq \|(\A^{\dagger}\A-\boldsymbol{I})\bc^{\star}\|_2 + \|\A^{\dagger}\|_2 \|\boldsymbol{h}\|_2.
$$
Note that when $m<N$, $\boldsymbol{I}-\A^\dagger\A$ is the projection onto the null space of $\A$ and thus its operator norm is bounded by $1$ and
\begin{align*}
      \|\bc^\sharp - \bc^*\|_2 
   \leq \|\bc^*\|_2 +  \|\A^\dagger\|_2\,\|\bh\|_2
    &\leq \|\bc^*\|_2 + \frac{1}{\sqrt{N}}\, \frac{1}{\sqrt{\lambda_{\text{min}}(N^{-1}\A\A^*)}}\,\|\bh\|_2\\
 \notag \\
    & \leq \frac{1}{\sqrt{N}}\|f\|_{\rho} + \frac{\sqrt{m}}{\sqrt{N}}{\frac{1}{\sqrt{1-3\eta}}}(\epsilon\|f\|_\rho + E)
\end{align*}
with probability at least $1-4\delta$, noting that $\left|\bc^\star_k\right| = N^{-1} \left|\frac{\alpha(\boldsymbol{\omega}_k)}{\rho(\boldsymbol{\omega}_k)}\right| \leq N^{-1} \, \|f\|_\rho$. In this regime, the $L^2$ generalization error is bounded by
\begin{align*}
    \|f&-f^{\sharp}\|_{L^2(d\mu)} \notag\\
    &\leq \|f-f^{\star}\|_{L^2(d\mu)} + \|f^{\star}-f^{\sharp}\|_{L^2(d\mu)} \notag \\
    &\leq \epsilon\|f\|_{\rho} + \|f^{\star} - f^{\sharp}\|_{L^2(d\mu)} \notag \\
    &\leq \epsilon\|f\|_{\rho} + \left(\frac{N}{m}\left({1+3\eta}\right)+N\sqrt{\frac{2}{m}\log\left(\frac{1}{\delta}\right)}\right)^{\frac{1}{2}}\|\bc^{\star}-\bc^{\sharp}\|_2 \notag \\
    &\leq \epsilon\|f\|_{\rho} \notag\\
    &\quad+ \left(m^{-\frac{1}{2}} {\sqrt{1+3\eta}} + m^{-\frac{1}{4}}\left(2\log\left(\frac{1}{\delta}\right)\right)^{\frac{1}{4}} \right)\left(\|f\|_{\rho} + m^{\frac{1}{2}}{\frac{1}{\sqrt{1-3\eta}}}(\epsilon\|f\|_\rho + E)\right) \notag \\
    &\leq \epsilon\|f\|_{\rho} + 2 m^{-\frac{1}{4}}\log^{\frac{1}{4}}\left(\frac{1}{\delta} \right) \|f\|_{\rho} + 2 m^{\frac{1}{4}} \,\sqrt{K(\eta)}\, \log^{\frac{1}{4}}\left(\frac{1}{\delta} \right) \left( \epsilon\|f\|_\rho + E\right) \notag \\
    &\leq \epsilon\|f\|_{\rho} + 2 \log^{\frac{1}{4}}\left(\frac{1}{\delta} \right) \left(m^{-\frac{1}{4}}+  \sqrt{K(\eta)}\, m^{\frac{1}{4}}\, \epsilon\right) \|f\|_{\rho} + 2 m^{\frac{1}{4}} \,\sqrt{K(\eta)}\, \log^{\frac{1}{4}}\left(\frac{1}{\delta} \right) E
\end{align*}
with probability at least $1-8\delta$.
\end{proof}

\begin{proof}[Proof of Theorem \ref{Th: error for bp}]
Following the proof of Theorem \ref{Th: error for least square m>N} and \ref{Th: error for least square m<N}, we have the analogous estimate
\begin{align*}
    \|f^*-f^{\sharp}\|_{L^2(d\mu)}^2 \leq \frac{1}{m}\sum_{j=1}^m |f^*(\boldsymbol{z}_j) - f^{\sharp}(\boldsymbol{z}_j)|^2+ \left(N\sqrt{\frac{2}{m}\log\left(\frac{1}{\delta}\right)}\right)\|\bc^*-\bc^{\sharp}\vert_{\mathcal{S}^{\sharp}}\|_2^2 \notag 
\end{align*}
with probability at least $1-\delta$. We apply Theorem \ref{thm: RIP estimate} with $s$ replaced by $2s$ and if the conditions of Theorem \ref{thm: RIP estimate} are satisfied (for example, by setting the parameters to $\eta_1 = 0.1$ and $\eta_2 = 0.2$), then 
$$
\delta_{2s}\left(\frac{1}{\sqrt{m}}\A\right)\leq \frac{4}{\sqrt{41}} \quad\text{ and }\quad \delta_{2s}\left(\frac{1}{\sqrt{m}}\boldsymbol{\tilde{A}}\right)\leq \frac{4}{\sqrt{41}}
$$
with probability at least $1-2\delta$. Following the proof of Lemma 2 in \cite{hashemi2021generalization}, the first term is bounded by
\begin{align*}
    m^{-\frac{1}{2}}\sqrt{\sum_{j=1}^m |f^{\star}(\boldsymbol{z}_j) - f^{\sharp}(\boldsymbol{z}_j)|^2} \leq 2\|\bc^{\star}-\bc^{\sharp}\vert_{\mathcal{S}^{\sharp}}\|_2 + 2\vartheta_{s,2}(\bc^{\star}) + \vartheta_{s,1}(\bc^{\star}).
\end{align*}
Then by the robust and stable recovery results, Lemma \ref{lm: robust recovery lemma}, we have
\begin{align*}
    \|\bc^{\star}-\bc^{\sharp}\vert_{\mathcal{S}^{\sharp}}\|_2 \leq C'\frac{\vartheta_{s,1}(\bc^{\star})}{\sqrt{s}}+C\xi +4\vartheta_{s,2}(\bc^{
    \star}).
\end{align*}
Combining these two results and Lemma \ref{lm: L2 error of f-f_star} yields (after possibly multiple redefinitions of the constants)
\begin{align*}
    \|f-f^{\sharp}\|&_{L^2(d\mu)} \notag\\
    &\leq \epsilon\|f\|_{\rho} + 2\left(1+N^{\frac{1}{2}}m^{-\frac{1}{4}}\log^{\frac{1}{4}}\left(\frac{1}{\delta}\right)\right)\|\bc^{\star}-\bc^{\sharp}\vert_{\mathcal{S}^{\sharp}}\|_2 + 2\vartheta_{s,2}(\bc^{\star}) + \vartheta_{s,1}(\bc^{\star}) \notag \\
    &\leq \epsilon\|f\|_{\rho} + C\left(1+N^{\frac{1}{2}}m^{-\frac{1}{4}}s^{-\frac{1}{2}}\log^{\frac{1}{4}}\left(\frac{1}{\delta}\right)\right)\vartheta_{s,1}(\bc^{\star}) \notag\\
    &\quad + C'\left(1+N^{\frac{1}{2}}m^{-\frac{1}{4}}\log^{\frac{1}{4}}\left(\frac{1}{\delta}\right)\right)\vartheta_{s,2}(\bc^{\star}) + C''\left(1+N^{\frac{1}{2}}m^{-\frac{1}{4}}\log^{\frac{1}{4}}\left(\frac{1}{\delta}\right)\right)\xi \notag \\
    &\leq C'\left(1 + N^{\frac{1}{2}}m^{-\frac{1}{4}}\log^{\frac{1}{4}}\left(\frac{1}{\delta}\right)\right)(\epsilon\|f\|_\rho + E) + C\left(1+N^{\frac{1}{2}}m^{-\frac{1}{4}}s^{-\frac{1}{2}}\log^{\frac{1}{4}}\left(\frac{1}{\delta}\right)\right)\vartheta_{s,1}(\bc^{\star}), \notag
\end{align*}
which concludes the proof.
\end{proof}

\section{Discussion}
\label{sec: discussion}

We analyze the double descent phenomenon \cite{belkin2020two} for random feature regression by relating it to the condition number of asymmetric rectangular random matrices and deriving (high probability) bounds on the extreme singular values. The technical arguments rely on random matrix theory, specifically, deriving a concentration bound on eigenvalues of the Gram matrix (or the restricted isometry constants) for the various complexity setting. The bounds improve on previous results in the literature and give a refined picture of the test error landscape as function of the complexity. In the interpolation regime, we directly derive a lower bound on the condition number, showing that the linear system becomes ill-conditioned when $N=m$. We provide risk bounds which are controlled by the conditioning of the random feature matrix, thereby relating the generalization error with the conditioning of the system. The risk bounds include the least squares, min-norm interpolation, and sparse regression problems. While our analysis focused on Fourier features with normally distributed weights and samples in order to provide neater bounds, the proofs could be extended to other features and probability distributions. For example, if we change the feature map, the constants in our results would need to include the $L^\infty$ norm of the new activation function. Additionally, using other probability distributions (e.g. uniform or subgaussian), would lead to changes in the constants and slight changes in the scalings. When $\boldsymbol{W}$ and $\X$ are sampled from different distributions, the symmetry between the underparameterized and overparameterized results breakdown but the main conclusions should still hold (i.e. constants may change, but the overall scaling should remain valid).

The connections between random feature models and fully trained neural networks are rich, and precisely translation our results in this setting is left as a future discussion. For deep networks, weight initialization and normalization can help to avoid the issue of vanishing gradients. That is, if the Jacobian of the hidden layers are close to being an isometry then the network is \textit{stable} in the dynamical sense \cite{pennington2017resurrecting, haber2017stable, zhang2020forward, sun2020neupde}. Based on our results, since the spectrum of random feature maps concentrate near 1, one could show that randomization helps to provide dynamic stability within each layer of a deep neural network.

A consequence of the theory presented in this work is that the random feature matrix does not need regularization to be well-conditioned and thus the risk bounds can decrease to zero without adding a penalty to the training problem. However, it may be the case that regularization improves generalization in the presence of noise and outliers. We leave a more detailed analysis of unconstrained regularized methods (for example, the ridge regression problem) with noisy data for future work.

\section*{Acknowledgement}
This work was supported in part by AFOSR MURI FA9550-21-1-0084 and NSF
DMS-1752116. The authors would also like to thank Rachel Ward for the helpful discussions.

\bibliographystyle{plain}

\clearpage
\newpage

\appendix
\appendixpage
\section{Key Lemmata}
The following lemma is used in the proof of Theorem \ref{thm: RIP estimate}. This is an improvement over Lemma 12.36 in \cite{foucart2013invitation}, where the constants, the scales in the logarithm factors, and $\tilde{C}_1$ are smaller.
\begin{lemma}\label{lm: expectation estimate}
Let $\X_1,\dots,\X_m$ be vectors in $\mathbb{C}^N$ with $\|\X_\ell\|_{\infty} \leq K$ for some $K>0$ and all $\ell\in [m]$. Then, for $s\leq m$, 
\begin{equation}
    \mathbb{E} \vertiii{\sum_{\ell=1}^m \epsilon_\ell \X_\ell \X_\ell^*} \leq \Tilde{C}_1 K \sqrt{s}\log(s)\sqrt{\log(3m)\log\left(3+\frac{N}{9\log(2m)}\right)}\sqrt{\vertiii{\sum_{\ell=1}^m \X_\ell\X_\ell^*}}
\end{equation}
where $\epsilon_\ell$ for $\ell\in [m]$ are independent Rademacher random variables and $\Tilde{C}_1>0$ is a universal constant. 
\end{lemma}

The remaining lemmata are used throughout this work.

\begin{lemma}[Corollary 8.15 from \cite{foucart2013invitation}]\label{Matrix Bernstein Ineq}
Let $\{\X_\ell\}_{\ell\in[m]}$ by a set of $\mathbb{C}^{N\times N}$ independent mean-zero self-adjoint random matrices, assume that $\|X_\ell\|_{2} \leq K$ almost surely for all $\ell\in[m]$, and define 
$$\sigma^2:= \left\| \sum_{\ell=1}^m \mathbb{E}(\X_\ell^2)\right\|.$$
 Then, for all $t>0$, 
 \begin{equation}
    \mathbb{P}\left(\left\| \sum_{\ell=1}^m \X_\ell \right\|_2 \geq t \right)\leq 
    2N \exp\left( -\frac{t^2}{2\sigma^2+\frac{2Kt}{3}}\right).
\end{equation}
\end{lemma}

\begin{lemma}[Theorem 8.42 and Remark 8.43(c) from \cite{foucart2013invitation}]\label{Bernstein Ineq}
Let $\mathcal{F}$ be a countable set of functions $f:\mathbb{C}^N \rightarrow \mathbb{R}$. Let $X_\ell$ for $\ell\in [m]$ be independent random vector in $\mathbb{C}^N$ such that $\mathbb{E}f(X_\ell) =0$ and  $f(X_\ell) \leq K$ a.s. for all $\ell\in [m]$ and for all $f \in \mathcal{F}$ for some constant $K>0$ and let
$$Z= \sup_{f\in \mathcal{F}} \,\left| \sum_{\ell=1}^m f(X_\ell)\right|.$$
 Let $\sigma^2_\ell$ be nonzero such that $\mathbb{E}[\left(f(X_\ell)\right)^2] \leq \sigma^2_\ell$ for all $f\in \mathcal{F}$ and $\ell\in [m]$. Then, for all $t>0$, 
 \begin{equation}
    \mathbb{P}\left(Z \geq \mathbb{E}Z +t \right)\leq 
    \exp\left( \frac{-t^2/2}{\sum_{\ell=1}^m \sigma^2_\ell+2K\mathbb{E}Z+tK/3}\right).
\end{equation}
\end{lemma}

\begin{lemma}[Theorem 9.24 in \cite{hashemi2021generalization}]\label{lm: robust recovery lemma}
Suppose that the $2s$-RIP constant of the matrix $\A\in\mathbb{C}^{m\times N}$ satisfies
$$
\delta_{2s}(\A)\leq \frac{4}{\sqrt{41}}
$$
then, for any vector $\bc^*\in\mathbb{C}^N$ satisfying $\boldsymbol{y}=\A\bc^* + \boldsymbol{e}$ with $\boldsymbol{e}\leq \xi\sqrt{m}$, a minimizer $\bc^{\sharp}$ of the BP problem \eqref{BP problem} approximates the vector $\bc^*$ with the error bounds
\begin{align*}
    \|\bc^*-\bc^{\sharp}\|_1 &\leq C'\kappa_{s,1}(\bc^*) + C\sqrt{s}\xi \\
    \|\bc^*-\bc^{\sharp}\|_2 &\leq C'\frac{\kappa_{s,1}(\bc^*)}{\sqrt{s}} + C\xi.
\end{align*}
where $C,C'>0$ are some constants. If we let $\mathcal{S}^{\sharp}$ be the index set of the $s$ largest (in magnitude) entries of $\bc^{
\sharp}$, then we have
\begin{align*}
    \|\bc^* - \bc^{\sharp}\vert_{
    \mathcal{S}^{\sharp}}\|_2 \leq C'\frac{\kappa_{s,1}(\bc^*)}{\sqrt{s}} + C\xi + 4\vartheta_{s,2}(\bc^*).
\end{align*}
\end{lemma}

\begin{lemma}[Lemma 8 in \cite{hashemi2021generalization}]\label{lm: samples lie within a ball}
Suppose that $\bx_j$ for $j\in[m]$ are sampled from the normal distribution $\mathcal{N}(\boldsymbol{0},\gamma^2\boldsymbol{I}_d)$. Let $\mathbb{B}^d(R)$ be the $\ell^2$-ball centered at $\boldsymbol{0}$ with radius $R>0$. Then for $\delta\in(0,1)$, the probability of $\bx_j\in \mathbb{B}^d(R)$ for all $j\in[m]$ is at least $1-\delta$ provided that
$$
R\geq \gamma \sqrt{d+\sqrt{12d\log(\frac{m}{\delta})}}.
$$
\end{lemma}

\begin{lemma}[Lemma 9 in \cite{hashemi2021generalization}]\label{lm: l_infty error of f-f_star}
Fix confidence parameter $\delta>0$ and accuracy parameter $\epsilon>0$. Suppose $f\in\mathcal{F}(\phi,\rho)$ where $\phi(
\bx,\boldsymbol{\omega}) = \exp(i\langle \bx,\boldsymbol{\omega} \rangle)$ and $\rho$ is a probability distribution with finite second moment used for sampling the random weights $\boldsymbol{\omega}$. Consider a set $\mathcal{X}\subset\mathbb{R}^d$ with diameter $R = \sup\limits_{\bx\in\mathcal{X}}\|\bx\|$. If the following holds
\begin{align*}
    N\geq \frac{4}{\epsilon^2}\left( 1+4R\sqrt{\mathbb{E}\|\boldsymbol{\omega}\|_2^2}+\sqrt{\frac{1}{2}\log\left(\frac{1}{\delta} \right)} \right)^2,
\end{align*}
then with probability at least $1-\delta$ with respect to the draw of $\boldsymbol{\omega}_k$ for $k\in[N]$, the following holds
\begin{align*}
    \sup_{\bx\in\mathcal{X}}|f(\bx)-f^*(\bx)|\leq \epsilon\|f\|_{\rho},
\end{align*}
where
\begin{align}\label{eq: def of f-star in lemma}
    f^*(\bx) = \sum_{k=1}^N c_k^* \exp(i\langle \bx, \boldsymbol{\omega}_k \rangle), \quad \text{with}\quad  c_k^* :=\frac{\alpha(\boldsymbol{\omega}_k)}{N\rho(\boldsymbol{\omega}_k)}.
\end{align}
\end{lemma}
\begin{lemma}[Lemma 1 in \cite{hashemi2021generalization}]\label{lm: L2 error of f-f_star}
Fix the confidence parameter $\delta>0$ and accuracy parameter $\epsilon>0$. Suppose $f\in\mathcal{F}(\phi,\rho)$ where $\phi(
\bx,\boldsymbol{\omega}) = \exp(i\langle \bx,\boldsymbol{\omega} \rangle)$ and $\rho$ is a probability distribution with finite second moment used for sampling the random weights $\boldsymbol{\omega}$. Suppose
\begin{align*}
    N\geq \frac{1}{\epsilon^2}\left( 1+\sqrt{2\log
    \left(\frac{1}{\delta}\right)} \right)^2,
\end{align*}
then with probability at least $1-\delta$ with respect to the draw of $\boldsymbol{\omega}_j$ for $j\in[N]$, the following holds
\begin{align*}
    \|f-f^*\|_{L^2(d\mu)}\leq \epsilon \|f\|_{\rho},
\end{align*}
where $f^*$ is defined as in \eqref{eq: def of f-star in lemma}.
\end{lemma}

\section{Proof of Lemma \ref{lm: expectation estimate}}

The proof of Lemma \ref{lm: expectation estimate} relies on a discrete version of Dudley's inequality. First, we introduce some notation. A stochastic process is a collection $\{X_t\}_{t\in T}$ of random variables indexed by some set $T$. The process $\{X_t\}_{t\in T}$ is called centered if $\mathbb{E}X_t = 0$ for all $t\in T$. The pseudometric $d$ associated to $\{X_t\}_{t\in T}$ is defined as
$$
d(s,t) = \sqrt{\mathbb{E}|X_s-X_t|^2}, \quad s,t\in T.
$$
A centered process is called subgaussian if 
$$
\mathbb{E}\exp(\theta(X_s-X_t))\leq \exp(\theta^2 d(s,t)^2/2), \quad s,t\in T,
$$
for some $\theta>0$.
For a set $T$, the covering number $\mathcal{N}(T,d,\delta)$ is the smallest integer $N$ such that there exists a set $A\subset T$ with cardinality $N$ and $\min\limits_{s\in F}d(t,s)\leq \delta$ for all $t\in T$.

\begin{lemma}[Discrete Dudley's inequality]\label{lm: discrete Dudley}
Let $\{X_t\}_{t\in T}$ be a centered subgaussian process with $T$ being a bounded set containing $0$ and $X_0 \equiv 0$. Let $\Delta(T) = \sup\limits_{t\in T}\sqrt{\mathbb{E}|X_t|^2}$ and $\delta_j = L^{-j} \Delta(T)$ for $j\in\mathbb{N}$, where $L>1$. Then for any integer $\alpha \geq 3$
\begin{align}
    \mathbb{E} \sup_{t\in T} X_t \leq \sqrt{2}\Bigg( &\sum_{j=\alpha}^{\infty}\sqrt{\log(\mathcal{N}(T,d,\delta_j))}\delta_{j-1} \notag\\
    + &\sum_{j=2}^{\alpha-1}\sqrt{\log(\mathcal{N}(B,d,\delta_j))}\delta_{j-1} +\frac{L}{L-1}\sqrt{\log(\mathcal{N}(B,d,\delta_1))}\Delta(T) \Bigg) \label{eq: discrete dudley for subgaussin process}\\
    \mathbb{E} \sup_{t\in T} |X_t| \leq \sqrt{2}\Bigg( &\sum_{j=\alpha}^{\infty}\sqrt{\log(2\mathcal{N}(T,d,\delta_j))}\delta_{j-1} \notag \\
    + &\sum_{j=2}^{\alpha-1}\sqrt{\log(2\mathcal{N}(B,d,\delta_j))}\delta_{j-1} +\frac{L}{L-1}\sqrt{\log(2\mathcal{N}(B,d,\delta_1))}\Delta(T) \Bigg) \label{eq: discrete dudley for abs subguassian process}.
\end{align}
Here $B$ is any set containing $T$.
\end{lemma}
\begin{remark}
The covering number is affected by the geometry of the set. It is possible that $\mathcal{N}(B,d,\delta) < \mathcal{N}(T,d,\delta)$ for $T\subset B$. In the standard version of this lemma, the value of $L$ is typically set to be $2$. However, other values lead to better constants in the final complexity bounds. Thus, we leave this as a free parameter in the lemma. 
\end{remark}
\begin{proof}
Let $T_j$ be the $\delta_j$-net of $T$ with $\card(T_j) = \mathcal{N}(T,d,\delta_j)$, and let $B_j$ be the $\delta_j$-net of $B$ with $\card(B_j) = \mathcal{N}(B,d,\delta_j)$. Define the mapping $\pi_j: T\to T_j$ and the mapping $\phi_j: B\to B_j$ so that
\begin{align*}
    &d(t,\pi_j(t))\leq \delta_j \quad \text{for all } t\in T\\
    &d(t,\phi_j(t))\leq \delta_j \quad \text{for all } t\in B.
\end{align*}
For any $n\in\mathbb{N}$ note that
$$
\max_{t\in T_n} X_t \leq \max_{t\in T_n}(X_t-X_{\pi_{n-1}(t)}) + \max_{t\in T_n} X_{\pi_{n-1}(t)} \leq \max_{t\in T_n}(X_t-X_{\pi_{n-1}(t)}) + \max_{t\in T_{n-1}} X_t,
$$
since $X_t = (X_t - X_{\pi_{n-1}(t)}) + X_{\pi_{n-1}(t)}$ for any $t$ and $n$. By Proposition 7.29 in \cite{foucart2013invitation} we have
$$
\mathbb{E} \max_{t\in T_n} X_t \leq \sqrt{2\log(\mathcal{N}(T,d,\delta_n))}\delta_{n-1} + \mathbb{E}\max_{t\in T_{n-1}} X_t.
$$
Thus for $n>\alpha$, by induction we obtain
\begin{equation}\label{eq: cut off dudley sum}
    \mathbb{E}\max_{t\in T_n} \leq \sum_{j=\alpha+1}^{n} \sqrt{2\log(\mathcal{N}(T,d,\delta_j))} \delta_{j-1} + \mathbb{E} \max_{t\in T_\alpha} X_t.
\end{equation}
If we view $t$ as a point in the set $B$, then the same argument gives
\begin{align*}
    \mathbb{E} \max_{t\in T_\alpha} X_t &\leq \mathbb{E} \max_{t\in T_\alpha} (X_t-X_{\phi_{\alpha-1}(t)}) + \mathbb{E} \max_{t\in T_\alpha} X_{\phi_{\alpha-1}(t)} \notag\\
    &\leq \sqrt{2\log(\mathcal{N}(T,d,\delta_\alpha))}\delta_{\alpha-1} + \mathbb{E}\max_{t\in T_\alpha} X_{\phi_{\alpha-1}(t)},
\end{align*}
and 
\begin{align*}
    \mathbb{E} \max_{t\in T_\alpha} X_{\phi_{\alpha-1}(t)} &\leq \mathbb{E} \max_{t\in T_\alpha} (X_{\phi_{\alpha-1}(t)} - X_{\phi_{\alpha-2}\circ \phi_{\alpha-1}(t)}) + \mathbb{E} \max_{t\in T_\alpha} X_{\phi_{\alpha-2}\circ \phi_{\alpha-1}(t)} \notag \\
    &\leq \mathbb{E} \max_{t\in B_{\alpha-1}} (X_t-X_{\phi_{\alpha-2}(t)}) + \mathbb{E} \max_{t\in T_\alpha} X_{\phi_{\alpha-2}\circ \phi_{\alpha-1}(t)} \notag \\
    &\leq \sqrt{2\log(\mathcal{N}(B,d,\delta_{\alpha-1}))}\delta_{\alpha-2} + \mathbb{E} \max_{t\in T_\alpha} X_{\phi_{\alpha-2}\circ \phi_{\alpha-1}(t)}.
\end{align*}
By induction, we obtain
\begin{align}\label{eq: replace the covering number}
    \mathbb{E} \max_{t\in T_\alpha} X_t \leq \sqrt{2\log(\mathcal{N}(T,d,\delta_\alpha))}\delta_{\alpha-1} + \sum_{j=2}^{\alpha} \sqrt{2\log(\mathcal{N}(B,d,\delta_j))}\delta_{j-1} + \mathbb{E} \max_{t\in T_\alpha} X_{\phi_1 \circ \phi_2 \circ \dots \circ \phi_{\alpha-1}(t)}.
\end{align}
To estimate the last term, we need to bound $d(\phi_1 \circ \phi_2 \circ \dots \circ \phi_{\alpha-1}(t), 0)$. Using the triangle inequality for $d$
\begin{align*}
    d(\phi_1 \circ \phi_2 \circ \dots \circ \phi_{\alpha-1}(t), 0) &\leq  d(\phi_1 \circ \phi_2 \circ \dots \circ \phi_{\alpha-1}(t), \phi_2 \circ \dots \circ \phi_{\alpha-1}(t)) \notag \\
    &\quad+ d(\phi_2 \circ \dots \circ \phi_{\alpha-1}(t),\phi_3 \circ \dots \circ \phi_{\alpha-1}(t)) \notag \\
    &\quad+ \dots + d(\phi_{\alpha-1}(t),t) + d(t,0) \notag \\
    &\leq \Delta(T) \sum_{j=0}^{\alpha-1} \frac{1}{L^j} \leq \Delta(T) \frac{L}{L-1}.
\end{align*}
Consequently,
\begin{equation}\label{eq: maximal inequality for the last setup of chaining}
    \mathbb{E} \max_{t\in T_\alpha} X_{\phi_1 \circ \phi_2 \circ \dots \circ \phi_{\alpha-1}(t)} \leq \sqrt{2\log(\mathcal{N}(B,d,\delta_1))} \Delta(T) \frac{L}{L-1}.
\end{equation}
Here we use the fact that $\card(\phi_1 \circ \phi_2 \circ \dots \circ \phi_{\alpha-1}(T_\alpha)) \leq \card(B_1)$.
Combining \eqref{eq: cut off dudley sum}, \eqref{eq: replace the covering number} and \eqref{eq: maximal inequality for the last setup of chaining} yields (for large $n$)
\begin{align}\label{eq: bound for t in T_n}
    \mathbb{E} \sup_{t\in T_n} X_t \leq \sqrt{2}\Bigg( &\sum_{j=\alpha}^{\infty}\sqrt{\log(\mathcal{N}(T,d,\delta_j))}\delta_{j-1} \notag\\
    + &\sum_{j=2}^{\alpha-1}\sqrt{\log(\mathcal{N}(B,d,\delta_j))}\delta_{j-1} +\frac{L}{L-1}\sqrt{\log(\mathcal{N}(B,d,\delta_1))}\Delta(T) \Bigg).
\end{align}
Note that the right-hand side does not depend on $n$. For any finite subset $A$ of $T$, we have
\begin{align*}
    \mathbb{E}\sup_{t\in A}X_t &\leq \mathbb{E} \sup_{t\in A} (X_t - X_{\pi_n(t)}) + \mathbb{E} \sup_{t\in A} X_{\pi_n(t)} \\
    &\leq \sqrt{2\log(\card(A))}\delta_n + \mathbb{E} \sup_{t\in T_n} X_t.
\end{align*}
The above inequality holds for any $n$, so we can let $n\to \infty$ and the first term goes to $0$ since $\delta_n \to 0$ as $n\to\infty$. This combined with \eqref{eq: bound for t in T_n} proves \eqref{eq: discrete dudley for subgaussin process}.

For \eqref{eq: discrete dudley for abs subguassian process}, the proof is the same, noting that for any finite subset $A$ of $T$, $$\mathbb{E}\max_{t\in A} |X_t| = \mathbb{E} \max_{t\in A} \{X_t, -X_t\}.$$
\end{proof}

\begin{proof}[Proof of Lemma \ref{lm: expectation estimate}]
By the definition of $\vertiii{\cdot}$,
\begin{equation}
    E:=\mathbb{E}\vertiii{\sum_{\ell = 1}^m \epsilon_\ell \X_\ell \X_\ell^*} = \mathbb{E} \sup_{\boldsymbol{z}\in D_{s,N}} \left| \sum_{\ell=1}^m \epsilon_\ell |\langle \X_\ell, \boldsymbol{z}\rangle|^2 \right|.
\end{equation}
To bound the expectation of the supremum of this Rademacher process, we will apply Lemma \ref{lm: discrete Dudley}. We need to estimate the covering number $\mathcal{N}(D_{s,N},d,t)$ for small $t$, and the covering number $\mathcal{N}(B,d,t)$ for large $t$ and some suitable set $B$ containing $D_{s,N}$. Here $d$ is the pseudometric associated with the stochastic process. We also need to estimate the quantity $\Delta(D_{s,N})$.

For $\boldsymbol{z}_1, \boldsymbol{z}_2\in D_{s,N}$,
\begin{align}\label{eq: bound pseudometric by seminorm}
    d(\boldsymbol{z}_1, \boldsymbol{z}_2) &= \sqrt{\sum_{\ell=1}^m \Big( |\langle\X_\ell,\boldsymbol{z}_1\rangle|^2 - |\langle\X_\ell,\boldsymbol{z}_2\rangle|^2 \Big)^2} \notag \\ 
    & = \sqrt{\sum_{\ell=1}^m \Big( (\boldsymbol{z}_1-\boldsymbol{z}_2)^* \X_\ell \X_\ell^* (\boldsymbol{z}_1+\boldsymbol{z}_2) \Big)^2} \notag \\
    &=\sqrt{\sum_{\ell=1}^m \big| \langle \X_\ell, \boldsymbol{z}_1-\boldsymbol{z}_2\rangle\big|^2 \big|\langle \X_\ell,\boldsymbol{z}_1+\boldsymbol{z}_2\rangle\big|^2} \notag \\
    &\leq \max_{\ell\in [m]} |\langle \X_\ell, \boldsymbol{z}_1-\boldsymbol{z}_2 \rangle| \sqrt{\sum_{\ell=1}^m |\langle \X_\ell, \boldsymbol{z}_1 + \boldsymbol{z}_2 \rangle|^2}.
\end{align}
using H\"{o}lder's inequality in the last step. Introduce the seminorm $\|\cdot\|_{X}$ on $\mathbb{C}^N$
\begin{equation*}
    \|\boldsymbol{z}\|_{X} := \max_{\ell\in [m]} |\langle \X_\ell, \boldsymbol{z}\rangle|
\end{equation*}
and let
\begin{equation*}
    R = \sup_{\boldsymbol{z}\in D_{s,N}} \sqrt{\sum_{\ell=1}^m |\langle \X_\ell ,\boldsymbol{z} \rangle|^2} = \sqrt{\vertiii{\sum_{\ell=1}^m \X_\ell \X_\ell^*}}.
\end{equation*}
Applying Minkowski's inequality to \eqref{eq: bound pseudometric by seminorm} yields $d(\boldsymbol{z}_1,\boldsymbol{z}_2) \leq 2R \|\boldsymbol{z}_1-\boldsymbol{z}_2\|_X$ for all $\boldsymbol{z}_1, \boldsymbol{z}_2 \in D_{s,N}$. Similarly, $d(\boldsymbol{z},0)\leq R\|\boldsymbol{z}\|_X$ for $z\in D_{s,N}$. Therefore, 
\begin{equation}\label{eq: radius of DsN}
    \Delta(D_{s,N}) = \sup_{\boldsymbol{z}\in D_{s,N}}d(\boldsymbol{z},0)\leq R \, \sup_{\boldsymbol{z}\in D_{s,N}
} \|\boldsymbol{z}\|_X\leq R\, \sup_{\boldsymbol{z}\in D_{s,N}} \max_{\ell\in [m]}\|\X_\ell\|_\infty \|\boldsymbol{z}\|_1 \leq RK\sqrt{s}.    
\end{equation}

Now we estimate the covering number $\mathcal{N}(D_{s,N}, \|\cdot\|_X,t)$ for small $t$. Since $\|\cdot\|_X\leq K\sqrt{s}\|\cdot\|_2$, Proposition C.3 from \cite{foucart2013invitation} gives
\begin{align}\label{eq: covering number for small t}
    \mathcal{N}(D_{s,N},\|\cdot\|_X,t) &\leq \sum_{S\subset [N], \card(S)=s} \mathcal{N}(B_{\|\cdot\|_2}^s, \|\cdot\|_2, \frac{t}{K\sqrt{s}}) \notag \\
    & \leq \binom{N}{s}\left( 1+\frac{2K\sqrt{s}}{t} \right)^{2s}\leq \left( \frac{eN}{s} \right)^s \left( 1+\frac{2K\sqrt{s}}{t} \right)^{2s},
\end{align}
where $B_{\|\cdot\|_2}^s$ is the unit $\ell^2$-ball in $\mathbb{C}^s$, thus
\begin{equation}\label{eq: bound for log(2N(DsN))}
    \log(2\mathcal{N}(D_{s,N},\|\cdot\|_X,t))\leq 2s\left(\log\left(\frac{eN}{s}\right)+\log\left(1+\frac{2K\sqrt{s}}{t}\right)\right).
\end{equation}
For large $t$, we first find a suitable set $B$ containing $D_{s,N}$. Define the norm $|\cdot|_1$ on $\mathbb{C}^N$ as
\begin{equation*}
    |\boldsymbol{z}|_1 = \sum_{j=1}^N |\text{Re}(z_j)| + |\text{Im}(z_j)|.
\end{equation*}
Then it follows from H\"{o}lder's inequality that
\begin{equation*}
    D_{s,N} \subset \sqrt{2s}B_{|\cdot|_1}^N = \{\boldsymbol{z}\in\mathbb{C}^N: |\boldsymbol{z}|_1 \leq \sqrt{2s}\}.
\end{equation*}
We then use Maurey's empirical method to estimate $\mathcal{N}(B_{|\cdot|_1}^N,\|\cdot\|_X,t)$. Note that the set $B_{|\cdot|_1}^N$ is the convex hull of $\{\pm\boldsymbol{e}_j, \pm i\boldsymbol{e}_j: j\in [N]\}$. If we enumerate these vectors as $\{\boldsymbol{v}_j\}_{j\in[4N]}$, then for any $\boldsymbol{z}\in  B_{|\cdot|_1}^N$, there are non-negative $\{\lambda_j\}_{j\in[4N]}$ with $\sum_{j=1}^{4N}\lambda_j = 1$ such that $\boldsymbol{z}=\sum_{j=1}^{4N}\lambda_j \boldsymbol{v}_j$. For this $\boldsymbol{z}$, we define a random variable $\boldsymbol{Y}$ which takes value $\boldsymbol{v}_j$ with probability $\lambda_j$ and let $\boldsymbol{Y}_1,\dots,\boldsymbol{Y}_M$ be i.i.d. copies, where $M$ is a positive integer to be determined. Since $\mathbb{E}\boldsymbol{Y}=\boldsymbol{z}$, if we let 
\begin{equation}\label{eq: M-sparse vector y}
\boldsymbol{y}=\frac{1}{M}\sum_{k=1}^M \boldsymbol{Y}_k, 
\end{equation}
by symmetrization we have
\begin{equation}
    \mathbb{E}\|\boldsymbol{y}-\boldsymbol{z}\|_X = \frac{1}{M}\mathbb{E}\|\sum_{k=1}^M (\boldsymbol{Y}_k - \mathbb{E}\boldsymbol{Y}_k)\|_X \leq \frac{2}{M} \mathbb{E} \|\sum_{k=1}^M \epsilon_k \boldsymbol{Y}_k\|_X.
\end{equation}
Here we use the fact that the seminorm $\|\cdot\|_X$ is convex.

For any realization of $\boldsymbol{Y}_k$ for $k\in[M]$ we set $S_\ell =\sum\limits_{k=1}^M \epsilon_k \langle \X_\ell, \boldsymbol{Y}_k \rangle $, then $\|\sum\limits_{k=1}^M \epsilon_k\boldsymbol{Y}_k\|_X = \max\limits_{\ell\in[m]} |S_\ell|$. Note that $|\langle \X_\ell, \boldsymbol{Y}_k\rangle|\leq \|\X_\ell\|_\infty \|\boldsymbol{Y}_k\|_1 \leq K$, by Theorem 8.8 in \cite{foucart2013invitation}, we have
\begin{equation}
    \mathbb{P}\left(|S_\ell|\geq \sqrt{M}K t\right) \leq 2\exp{\left(-\frac{t^2}{2}\right)},\quad t>0.
\end{equation}
and by the union bound
\begin{equation}
    \mathbb{P}\left(\max_{\ell\in[m]}|S_\ell|\geq \sqrt{M}Kt\right)\leq 2m\exp{\left(-\frac{t^2}{2}\right)}\quad t>0.
\end{equation}
Thus, for $m\geq3$ (otherwise this would be uninteresting)
\begin{align}\label{eq: bound for expectation of max of rademacher sum}
    \mathbb{E}\max_{\ell\in[m]}|S_\ell| &= \int_0^\infty \mathbb{P}(\max_{\ell\in[m]}|S_\ell|\geq t) dt \notag\\
    &\leq \int_{0}^{K\sqrt{2M\log(2m)}} dt +2m \int_{K\sqrt{2M\log(2m)}}^\infty e^{-\frac{t^2}{2MK^2}}dt \notag\\
    &=K\sqrt{2M\log(2m)} + 2mK\sqrt{\frac{M\pi}{2}}\left(1-\text{erf}(\sqrt{\log(2m)})\right) \notag \\
    &=\frac{5}{4}K\sqrt{2M\log(2m)}.
\end{align}
The above argument holds for any realization of 
$\{\boldsymbol{Y}_k\}_{k\in[M]}$. By Fubini's theorem, we have
\begin{equation}
    \mathbb{E}\|\boldsymbol{y}-\boldsymbol{z}\|_X \leq \frac{2}{M} \mathbb{E}_{\boldsymbol{Y}} \mathbb{E}_{\epsilon} \left\|\sum_{k=1}^M \epsilon_k \boldsymbol{Y}_k\right\|_X =\frac{2}{M} \mathbb{E}_{\boldsymbol{Y}}\mathbb{E}_{
    \epsilon} \max_{\ell\in [m]} |S_\ell| \leq \frac{5K}{\sqrt{2M}}\sqrt{\log(2m)}.
\end{equation}
Therefore, for this $\boldsymbol{z}\in B_{|\cdot|_1}^N$, we can find a $\boldsymbol{y}$ of the form \eqref{eq: M-sparse vector y} such that $\|\boldsymbol{z}-\boldsymbol{y}\|_X\leq 5K\sqrt{\log(2m)}/\sqrt{2M}$. Thus for any $t>0$, if 
$$
M\geq\frac{25K^2}{2t^2}\log(2m),
$$
then $\|\boldsymbol{z}-\boldsymbol{y}\|_X<t$, therefore, we can set
$$
M = \left\lfloor \frac{25K^2}{2t^2}\log(3m) \right\rfloor.
$$
since we only need to consider $t\leq \sqrt{2}K/4$. The set
$$
\mathcal{Y}=\left\{\boldsymbol{y}=\frac{1}{M}\sum_{k=1}^M\boldsymbol{Y}_k: \boldsymbol{Y}_k\in \{\pm\boldsymbol{e}_1,\dots,\pm\boldsymbol{e}_N,\pm i\boldsymbol{e}_1,\dots,\pm i\boldsymbol{e}_N\}\right\}
$$
has cardinality at most $\binom{M+4N-1}{M}<(e(1+4N/M))^M$ and
is a $t$-covering for $B_{|\cdot|_1}^N$. Note that here $M\geq 100\log(2m)$ since $t\leq \sqrt{2}K/4$. Therefore
\begin{equation}\label{eq: bound for log(2N(B))}
    \log(2\mathcal{N}(B_{|\cdot|_1^N},\|\cdot\|_X,t))\leq M\log\left(3+\frac{12N}{M}\right)\leq \frac{25K^2}{2t^2}\log(3m) \log\left(3+ \frac{N}{9\log(2m)}\right).
\end{equation}
With the estimates \eqref{eq: radius of DsN}, \eqref{eq: bound for log(2N(DsN))} and \eqref{eq: bound for log(2N(B))}, we apply Lemma \ref{lm: discrete Dudley}. Let $L = e$, with $\delta_j = L^{-j} RK\sqrt{s}$ and set $\mathcal{C} = 5\sqrt{\log(3m)\log(3+\frac{N}{9\log(2m)})/2}$
\begin{align*}
    E \leq& \sqrt{2}\Bigg( \sum_{j=\alpha}^{\infty}\sqrt{\log(2\mathcal{N}(D_{s,N},d,\delta_j))}\delta_{j-1} \notag \\
    &+ \sum_{j=2}^{\alpha-1}\sqrt{\log(2\mathcal{N}(\sqrt{2s}B_{|\cdot|_1}^N,d,\delta_j))}\delta_{j-1} +\frac{L}{L-1}\sqrt{\log(2\mathcal{N}(\sqrt{2s}B_{|\cdot|_1^N},d,\delta_1))}\delta_0 \Bigg) \notag \\
    \leq& \sqrt{2}\Bigg( \sum_{j=\alpha}^{\infty}\sqrt{\log(2\mathcal{N}(D_{s,N},\|\cdot\|_X,\frac{\delta_j}{2R}))}\delta_{j-1} \notag \\
    \quad \quad &+ \sum_{j=2}^{\alpha-1}\sqrt{\log(2\mathcal{N}(B_{|\cdot|_1}^N,\|\cdot\|_X,\frac{\delta_j}{2\sqrt{2s}R}))}\delta_{j-1} +\frac{L}{L-1}\sqrt{\log(2\mathcal{N}(B_{|\cdot|_1^N},\|\cdot\|_X,\frac{\delta_1}{2\sqrt{2s}R}))}\delta_0 \Bigg) \notag \\
    \leq& \sqrt{2}\left(\sqrt{2s}\sum_{j=\alpha}^\infty \sqrt{\log(\frac{eN}{s}) + \log(1+4L^j)}\frac{RK\sqrt{s}}{L^{j-1}} + \sum_{j=2}^{\alpha-1} 2\sqrt{2s}RK L \mathcal{C} + \frac{L}{L-1} 2\sqrt{2s} RK L \mathcal{C} \right) \notag \\
    \leq& 2\sqrt{s} \left(\sqrt{\log\left(\frac{eN}{s}\right)} \frac{2RK\sqrt{s}}{e^{\alpha-1}} +  \frac{2\sqrt{\log(1+4e^\alpha)}RK\sqrt{s}}{e^{\alpha-1}}  \right) + 4\sqrt{s}RKe\mathcal{C}(\alpha-2) + \frac{4e}{e-1} \sqrt{s} RKe\mathcal{C}.
\end{align*}
We use the inequality $\sqrt{a+b}\leq\sqrt{a}+\sqrt{b}$ and $\sum_{j=\alpha}^\infty \frac{\sqrt{\log(1+4e^j)}}{e^{j-1}} \leq \frac{2\sqrt{\log(1+4e^\alpha)}}{e^{\alpha-1}}$ in the last inequality. Now choose $\alpha$ to be an integer such that $e^{\alpha-2}\leq \sqrt{s}\leq e^{\alpha-1}$, we have $\frac{1}{e^{\alpha-1}}\leq \frac{1}{\sqrt{s}}$ and $\alpha-2 \leq \log(\sqrt{s})$. Assume $s\geq 4$ we obtain
\begin{align*}
    E &\leq 4RK\sqrt{s}\left( \sqrt{\log\left(\frac{eN}{s}\right)} + \sqrt{\log(1+30\sqrt{s})} + \frac{47e}{8} \log(\sqrt{s}) \sqrt{\log(3m)\log\left(3+\frac{N}{9\log(2m)}\right)} \right) \notag\\
    &\leq \Tilde{C}_1 K\sqrt{s} \log(s) \sqrt{\log(3m)\log\left(3+\frac{N}{2\log(2m)}\right)}\sqrt{\vertiii{\sum_{\ell=1}^m \X_\ell \X_\ell^*}} ,
\end{align*}
where
$$
\Tilde{C}_1 \leq \frac{4}{\log(4)} + \frac{6}{\sqrt{\log(9)\log(3+4/\log(6))}} + \frac{47e}{4} \approx 37.97,
$$
which concludes the proof.
\end{proof}
\begin{remark}
The constant in \eqref{eq: bound for expectation of max of rademacher sum} approaches $1$ as $m \to \infty$, in the sense that if a lower bound is assigned for $m$ we can choose a particular value of the constant. For reference, the constant can be set to $8/7$ when $m\geq 12$ and is $11/10$ when $m\geq 49$. Thus in the practical regime, where $m$ is large, the effective value is close to $1$. 
\end{remark}
\begin{remark}
The value of $\Tilde{C}_1$ in this lemma can be made smaller when $m$ and $s$ (or $N$) are large. In the asymptotic sense, when the dimensional parameters are large, the value of  $\Tilde{C}_1$  limits to $4\sqrt{2}e \approx 15.38$.
\end{remark}

\section{Proof of Theorem \ref{thm: RIP estimate}}
For a subset $S\subset[N]$, let $\A_S$ denote the submatrix of $\A$ consisting of only the columns indexed by the set $S$. The $s$-th RIP constant $\delta_s(\A)$ can also be characterized by the maximization problem 
\begin{equation*}
    \delta_s(\A)= \max_{S\subset [N], \card(S)=s} \|\A_S^* \A_S - \boldsymbol{I}_s \|_2.
\end{equation*}
To measure the sparsity, we denote $\|z\|_0$ to be the number of nonzero entries of the vector $z$. Let $D_{s,N}$ be the union of all unit balls of dimension $s$ in the ambient dimension $N$ i.e.
\begin{equation*}
    D_{s,N} := \{\boldsymbol{z}\in \mathbb{C}^N: \|\boldsymbol{z}\|_2\leq 1, \|\boldsymbol{z}\|_0\leq s \},
\end{equation*}
define the following seminorm on $\mathbb{C}^{N\times N}$
\begin{equation*}
    \vertiii{\boldsymbol{B}} := \sup_{\boldsymbol{z}\in D_{s,N}} |\langle \boldsymbol{B}\boldsymbol{z},\boldsymbol{z} \rangle|,
\end{equation*}
and thus the RIP constant $\delta_s(\A) = \vertiii{\A^*\A - \boldsymbol{I}_N}$ \cite{foucart2013invitation}. 
\begin{proof}[Proof of Theorem \ref{thm: RIP estimate}]
Parts of the proof follow from the general idea of the proof of Theorem 12.32 in \cite{foucart2013invitation}; however, many key steps differ since random feature matrices do not form orthogonal systems and since the elements of a random feature matrix are not independent.

Denote the $\ell$-th column of $\A^*$ by $\X_\ell = [\overline{\phi(\bx_\ell;\boldsymbol{\omega}_1)},\dots,\overline{\phi(\bx_\ell;\boldsymbol{\omega}_N)}]^T\in\mathbb{C}^N$. Then $\A^*\A = \sum\limits_{\ell=1}^m \X_\ell\X_\ell^*$ and the $s$-th RIP constant $\delta_s\left(\frac{1}{\sqrt{m}}\A\right)$ is bounded by
\begin{align*}
\delta_s\left(\frac{1}{\sqrt{m}}\A\right) &= \vertiii{\frac{1}{m}\sum_{\ell=1}^m \X_\ell\X_\ell^* - \boldsymbol{I}_N} \notag\\
&= \frac{1}{m}\vertiii{\sum_{\ell=1}^m (\X_\ell\X_\ell^* - \mathbb{E}_{\bx}(\X_\ell\X_\ell^*) +\mathbb{E}_{\bx}(\X_\ell\X_\ell^*)- \mathbb{E}_{\bx,\bomega}(\X_\ell\X_\ell^*) + \mathbb{E}_{\bx,\bomega}(\X_\ell\X_\ell^*)-\boldsymbol{I}_N)} \notag \\
&\leq \frac{1}{m} \vertiii{\sum_{\ell=1}^m (\X_\ell\X_\ell^* - \mathbb{E}_{\bx}\X_\ell\X_\ell^*)} + \frac{1}{m} \vertiii{\sum_{\ell=1}^m (\mathbb{E}_{\bx}(\X_\ell\X_\ell^*) - \mathbb{E}_{\bx,\bomega}\X_\ell\X_\ell^*)}+ \vertiii{\boldsymbol{L}},
\end{align*}
where $\boldsymbol{L} \in \mathbb{C}^{N\times N}$ is a matrix with $L_{j,j}=0$ and $L_{j,k}=(2\gamma^2\sigma^2+1)^{-\frac{d}{2}}$ for $j,k\in[N],j\neq k$. 

For any $\boldsymbol{z}\in D_{s,N}$
\begin{equation}
   \left|\langle \boldsymbol{L}\boldsymbol{z},\boldsymbol{z} \rangle\right| = \left|\sum_{\substack{j,k\in \supp(\boldsymbol{z})\\ j\neq k}} \left(\frac{1}{2\gamma^2\sigma^2+1}\right)^{\frac{d}{2}} \, z_j\, \overline{z_k}\right| \leq  \left(\frac{1}{2\gamma^2\sigma^2+1}\right)^{\frac{d}{2}}\,\|\boldsymbol{z}\|_1^2 \leq s\, \left(\frac{1}{2\gamma^2\sigma^2+1}\right)^{\frac{d}{2}},
\end{equation}
and using the complexity assumption yields 
$$\vertiii{\boldsymbol{L}} \leq s(2\gamma^2\sigma^2+1)^{-\frac{d}{2}}\leq \eta_1.$$

Note that by the definition of $\vertiii{\cdot}$,  $\vertiii{\boldsymbol{B}}\leq \|\boldsymbol{B}\|_2$ for any self-adjoint matrix $\boldsymbol{B}$. From the proof of Theorem \ref{Th: Main Theorem (simplified)}, we have $\vertiii{\mathbb{E}_{\bx}(\X_{\ell}\X_{\ell}^*) - \mathbb{E}_{\bx,\bomega}(\X_{\ell}\X_{\ell}^*)}\leq \eta_1$ with probability at least $1-\delta$ with respect to the draw of $\{\bomega_\ell\}_{\ell\in [N]}$ if the complexity assumption holds.

For the remaining arguments, we will use that  $\vertiii{\mathbb{E}_{\bx}(\X_{\ell}\X_{\ell}^*) - \mathbb{E}_{\bx,\bomega}(\X_{\ell}\X_{\ell}^*)}\leq \eta_1$ conditioned on the given draw of $\{\bomega_k\}_{k\in [N]}$. Note that $\vertiii{\cdot}$ is a convex function (since it is a seminorm) and thus symmetrization yields
\begin{equation}
    \mathbb{E}_{\bx} \vertiii{\sum_{\ell=1}^m (\X_\ell\X_\ell^* - \mathbb{E}_{\bx}(\X_\ell\X_\ell^*))} \leq 2 \mathbb{E}_{\bx,\epsilon} \vertiii{\sum_{\ell=1}^m \epsilon_\ell \X_\ell \X_\ell^*}
\end{equation}
where we introduce the Rademacher random variables $\epsilon_\ell$ for $\ell\in [m]$ which are independent of data and weights. Applying Fubini's theorem and Lemma \ref{lm: expectation estimate}, we obtain
\begin{align*}
    E:&= \frac{1}{m} \mathbb{E}_{\bx}\vertiii{\sum_{\ell=1}^m (\X_\ell\X_\ell^* - \mathbb{E}_{\bx}(\X_{\ell}\X_{\ell}^*))} \notag\\
    & \leq \frac{2}{m} \mathbb{E}_{\bx,\epsilon} \vertiii{\sum_{\ell=1}^m \epsilon_\ell \X_\ell\X_\ell^*} \notag\\
    &\leq \frac{2}{m} \Tilde{C}_1 \sqrt{s} \log(s) \sqrt{\log(3m)\log\left(3+\frac{N}{9\log(2m)}\right)}\, \mathbb{E}_{\bx}  \sqrt{\vertiii{\sum_{\ell=1}^m \X_\ell\X_\ell^*}}  \notag \\
    &= 2\Tilde{C}_1 \sqrt{s} \log(s) \sqrt{\frac{\log(3m)}{m}\log\left(3+\frac{N}{9\log(2m)}\right)}\, \mathbb{E}_{\bx} \sqrt{\vertiii{\frac{1}{m}\sum_{\ell=1}^m \X_\ell\X_\ell^* }}  \notag \\
    &\leq 2\Tilde{C}_1 \sqrt{s} \log(s) \sqrt{\frac{\log(3m)}{m}\log\left(3+\frac{N}{9\log(2m)}\right)} \sqrt{\mathbb{E}_{\bx} \vertiii{\frac{1}{m}\sum_{\ell=1}^m \left(\X_\ell\X_\ell^* -\mathbb{E}_{\bx}(\X_{\ell}\X_{\ell}^*)\right) } + 1+2\eta_1}.
\end{align*}
where we used Jensen's inequality and the estimate 
\begin{align*}
    \vertiii{\mathbb{E}_{\bx}(\X_{\ell}\X_{\ell}^*)} &\leq \vertiii{\mathbb{E}_{\bx}(\X_{\ell}\X_{\ell}^*) - \mathbb{E}_{\bx,\bomega}(\X_{\ell}\X_{\ell}^*)} + \vertiii{\mathbb{E}_{\bx,\bomega}(\X_{\ell}\X_{\ell}^*)-\boldsymbol{I}_N} + \vertiii{\boldsymbol{I}_N} \\
    &= \vertiii{\mathbb{E}_{\bx}(\X_{\ell}\X_{\ell}^*) - \mathbb{E}_{\bx,\bomega}(\X_{\ell}\X_{\ell}^*)} + \vertiii{\boldsymbol{L}} + \vertiii{\boldsymbol{I}_N} \\
    &\leq 2\eta_1+1
\end{align*}
in the last inequality. If set $D = 2\Tilde{C}_1 \sqrt{s} \log(s) \sqrt{\frac{\log(3m)}{m} \log\left(3+\frac{N}{9\log(2m)}\right)}$, the above inequality becomes $E\leq D\sqrt{E+1+2\eta}$. Rearranging the terms yields the following inequality (assuming that $\eta_1\leq \frac{1}{2}$)
\begin{equation*}
    E\leq \frac{D^2}{2}+\sqrt{\frac{D^4}{4}+2D^2} \leq D^2+\sqrt{2}D.
\end{equation*}
Next we use Bernstein's inequality for suprema of an empirical process (Lemma \ref{Bernstein Ineq}) to show that $m^{-1}\vertiii{\sum_{\ell=1}^m(\X_{\ell}\X_{\ell}^* - \mathbb{E}_{\bx}(\X_{\ell}\X_{\ell}^*))}$ is small with high probability. By the definition of the seminorm $\vertiii{\cdot}$, we have
\begin{align*}
    \vertiii{\sum_{\ell=1}^m (\X_\ell\X_\ell^* - \mathbb{E}_{\bx}(\X_{\ell}\X_{\ell}^*))} 
    &= \sup_{\boldsymbol{z}\in D_{s,N}} \left|\sum_{\ell=1}^m \langle (\X_\ell\X_\ell^* - \mathbb{E}_{\bx}(\X_{\ell}\X_{\ell}^*))\boldsymbol{z}, \boldsymbol{z} \rangle \right| \\
    &= \sup_{\boldsymbol{z}\in D_{s,N}^*} \left|\sum_{\ell=1}^m \langle (\X_\ell\X_\ell^* - \mathbb{E}_{\bx}(\X_{\ell}\X_{\ell}^*))\boldsymbol{z}, \boldsymbol{z} \rangle \right|,
\end{align*}
where $D_{s,N}^*$ is a countable dense subset of $D_{s,N}$, which exists since $D_{s,N}$  is the finite union of separable sets (unit balls in $\mathbb{C}^s$ as a subspace of $\mathbb{C}^N$ with its norm).

For $\boldsymbol{z}\in D_{s,N}^*$, define the random variable $f_{\boldsymbol{z}}(\X_{\ell}) = \langle (\X_{\ell}\X_{\ell}^* - \mathbb{E}_{\bx}(\X_{\ell}\X_{\ell}^*))\boldsymbol{z}, \boldsymbol{z}\rangle$. Note that $\mathbb{E}_{\bx} f_{\boldsymbol{z}}(\X_{\ell})=0$. To apply Lemma \ref{Bernstein Ineq}, we must first show that $f_{\boldsymbol{z}}(\X_\ell)$ is bounded and compute its variance. The random variables $f_{\boldsymbol{z}}(\X_\ell)$ for $\ell\in[m]$ are uniformly bounded by $2s$ since
\begin{align*}
    |f_{\boldsymbol{z}}(\X_\ell)| &\leq \sum_{j,k\in S} \left|\exp(i\langle \bx_\ell, \boldsymbol{\omega}_k - \boldsymbol{\omega}_j \rangle) - \exp\left(-\frac{\gamma^2}{2}\|\bomega_k-\bomega_j\|_2^2\right)\right|\left|\overline{z}_j z_k\right| \notag \\
    &\leq 2\|\boldsymbol{z}\|_1^2 \leq 2s,
\end{align*}
where $S=\supp(\boldsymbol{z})$. The variance of $f_{\boldsymbol{z}}(\X_\ell)$ for  $\ell\in[m]$ is bounded by
\begin{align*}
    \mathbb{E}_{\bx}[f_{\boldsymbol{z}}(\X_\ell)]^2 &= \mathbb{E}_{\bx}|\langle (\X_\ell\X_\ell^* - \mathbb{E}_{\bx}(\X_\ell\X_\ell^*))\boldsymbol{z}, \boldsymbol{z} \rangle|^2 \notag\\
    &\leq \mathbb{E}_{\bx}\|(\X_{\ell,S}\X_{\ell,S}^*-\mathbb{E}_{\bx}(\X_\ell\X_\ell^*))\boldsymbol{z}\|_2^2  \notag \\
    & = \mathbb{E}_{\bx}[\boldsymbol{z}^*\X_{\ell,S}\X_{\ell,S}^*\X_{\ell,S}\X_{\ell,S}^*\boldsymbol{z}] - \boldsymbol{z}^*\mathbb{E}_{\bx}(\X_\ell\X_\ell^*)\mathbb{E}_{\bx}(\X_\ell\X_\ell^*)\boldsymbol{z} \notag \\
    &\leq s \langle \mathbb{E}_{\bx}(\X_\ell\X_\ell^*)\boldsymbol{z}, \boldsymbol{z} \rangle \notag \\
    &\leq s \vertiii{\mathbb{E}_{\bx}(\X_\ell\X_\ell^*)}\leq 2s,
\end{align*}
using the fact that $\|\X_{\ell, S}\|_2^2 = s$ and the estimate $\vertiii{\mathbb{E}_{\bx}(\X_\ell\X_\ell^*)}\leq 1+2\eta_1\leq 2$.

If we set
\begin{align}
    D &= 2\Tilde{C}_1 \sqrt{s} \log(s) \sqrt{\log\left(3+\frac{N}{9\log(2m)}\right)}\sqrt{\frac{\log(3m)}{m}}\leq \eta_2 \label{eq: restriction on D}
\end{align}
for some $\eta_2 \in (0,\frac{1}{2})$, then
\begin{align}
    &\mathbb{P}_{\bx}\left(\vertiii{\sum_{\ell=1}^m(\X_\ell\X_\ell^*-\mathbb{E}_{\bx}(\X_\ell\X_\ell^*))}\geq m(\eta_2^2+\sqrt{2}\eta_2)+m\eta_1\right) \notag \\
    \leq & \mathbb{P}_{\bx}\left(\vertiii{\sum_{\ell=1}^m(\X_\ell\X_\ell^*-\mathbb{E}_{\bx}(\X_\ell\X_\ell^*))}\geq \mathbb{E}_{\bx}\vertiii{\sum_{\ell=1}^m(\X_\ell\X_\ell^*-\mathbb{E}_{\bx}(\X_\ell\X_\ell^*))}+m\eta_1\right) \notag \\
    \leq & \mathbb{P}_{\bx}\left(\sup_{\boldsymbol{z}\in D_{s,N}^*}\left|\sum_{\ell=1}^m f_{\boldsymbol{z}}(\X_{\ell})\right|\geq \mathbb{E}_{\bx}\sup_{\boldsymbol{z}\in D_{s,N}^*}\left|\sum_{\ell=1}^m f_{\boldsymbol{z}}(\X_{\ell})\right|+m\eta_1\right) \notag \\
    \leq & \exp\left(\frac{-m\eta_1^2/2}{2s + 4s(\eta_2^2+\sqrt{2}\eta_2)+2s\eta/3}\right) \label{eq: probability from Bernstein ineq},
\end{align}
where the last inequality uses Lemma \ref{Bernstein Ineq} with $K=2s$, $\sum_{\ell=1}^m \sigma_{\ell}^2 \leq 2ms$, $\mathbb{E}_{\bx}(Z)\leq m(\eta_2^2+\sqrt{2}\eta_2)$ and $t=m\eta_1$. Thus, 
\begin{align}
    \mathbb{P}_{\bx}\left(\frac{1}{m}\vertiii{\sum_{\ell=1}^m (\X_{\ell}\X_{\ell}^* - \mathbb{E}_{\bx}(\X_{\ell}\X_{\ell}^*))}\leq \eta_2^2+\sqrt{2}\eta_2+\eta_1\right)> 1-\delta
\end{align}
if the following conditions hold (assuming $\eta_2^2+\sqrt{2}\eta_2\leq 1$ and $\eta_1\leq \frac{1}{2}$)
\begin{align*}
    m&\geq C_1 \eta_1^{-2} s \log(\delta^{-1}) \\
    \frac{m}{\log(3m)} &\geq C_2 \eta_2^{-2} s \log^2(s) \log\left(3+\frac{N}{9\log(2m)}\right). 
\end{align*}
The universal constants satisfy $C_1\leq 13$ and $C_2 = 4\Tilde{C}_1^2$ where $\Tilde{C}_1$ is the value from Lemma \ref{lm: expectation estimate}.

Altogether, we have 
\begin{align*}
\delta_s\left(\frac{1}{\sqrt{m}}\A\right) &= \frac{1}{m}\vertiii{\sum_{\ell=1}^m (\X_\ell\X_\ell^* - \mathbb{E}_{\bx}(\X_\ell\X_\ell^*) +\mathbb{E}_{\bx}(\X_\ell\X_\ell^*)- \mathbb{E}_{\bx,\bomega}(\X_\ell\X_\ell^*) + \mathbb{E}_{\bx,\bomega}(\X_\ell\X_\ell^*)-\boldsymbol{I}_N)} \notag \\
&\leq \frac{1}{m} \vertiii{\sum_{\ell=1}^m (\X_\ell\X_\ell^* - \mathbb{E}_{\bx}\X_\ell\X_\ell^*)} + \frac{1}{m} \vertiii{\sum_{\ell=1}^m (\mathbb{E}_{\bx}(\X_\ell\X_\ell^*) - \mathbb{E}_{\bx,\bomega}\X_\ell\X_\ell^*)}+ \vertiii{\boldsymbol{L}} \notag \\
&\leq \eta_2^2+\sqrt{2}\eta_2 + 3\eta_1,
\end{align*}
with probability at least $1-2\delta$ if the conditions in the theorem are satisfied, which completes the proof.
\end{proof}

\end{document}